\documentclass[accepted]{uai2023} 

\usepackage[american]{babel}

\usepackage{natbib} 
\usepackage{amsmath,amssymb,amsfonts,amsthm}
\usepackage{algorithm,algorithmic}
\usepackage{mathtools, bbm}
\usepackage[mathscr]{euscript}
\usepackage{dsfont}
\usepackage{booktabs,multirow}
\usepackage{nicefrac}
\usepackage{subfig}
\newtheorem{definition}{Definition}
\newtheorem{proposition}{Proposition}
\newtheorem{lemma}{Lemma}
\newtheorem{assumption}{}

\newtheorem{theorem}{Theorem}
\newtheorem{corollary}{Corollary}
\newtheorem{remark}{Remark}
\newcommand{\R}{\mathbb{R}}
\newcommand{\E}{\mathbb{E}}

\newcommand{\p}{\mathbb{P}}
\newcommand{\1}{\mathds{1}}
\allowdisplaybreaks

\title{A policy gradient approach for optimization of smooth risk measures}

\author[1]{Nithia Vijayan}
\author[1]{Prashanth L.A.}
\affil[1]{%
    Department of Computer Science and Engineering,
    Indian Institute of Technology Madras, India.
}

\begin{document}
\maketitle
\begin{abstract}
We propose policy gradient algorithms for solving a risk-sensitive reinforcement learning (RL) problem in on-policy as well as off-policy settings. We consider episodic Markov decision processes, and model the risk using the broad class of smooth risk measures of the cumulative discounted reward. We propose two template policy gradient algorithms that optimize a smooth risk measure in on-policy and off-policy RL settings, respectively. We derive non-asymptotic bounds that  quantify the rate of convergence of our proposed algorithms to a stationary point of the smooth risk measure. As special cases, we establish that our algorithms apply to optimization of mean-variance and distortion risk measures, respectively.
\end{abstract}
\section{Introduction}
\label{sec:intro}
Risk-sensitive reinforcement learning (RL) has received a lot of attention recently in the literature, and a few representative works are  \cite{tamar2012policy,prashanth2014cvar,tamar2015coherent,borkar2010risk,chow2017risk,prashanth2016mlj,borkar2010learning,prashla16,huang2017risk}. Mean-variance tradeoff \cite{markowitz1952portfolio}, value at risk (VaR), conditional value at risk (CVaR) \cite{rockafellar2000optimization}, spectral risk measure \cite{acerbi2002spectral}, distortion risk measure \cite{denneberg1990}, a risk measure based on cumulative prospect theory (CPT) \cite{tversky1992advances} are some of the popular risk measures considered in the literature.

Policy gradients form a popular solution approach for traditional risk-neutral RL. The idea here is to consider a parameterized set of policies, usually in a continuous space, and perform a random search using stochastic gradient ascent to find a `good-enough' policy that optimizes a certain performance criterion. Several risk-sensitive RL algorithms employ this approach to find policies that are risk-optimal, see \cite{prashla2021} for a detailed survey of some of the recent developments in this research direction.

In this paper, we consider the problem of optimizing an abstract smooth risk measure (SRM) in a risk-sensitive RL context. SRMs constitute a broad class of risk measures that includes mean-variance risk measure (MVRM) and distortion risk measure (DRM). Mean-variance tradeoff is a well-known risk measure that is closely related to exponential cost risk measure -- a connection that can be seen using a Taylor series expansion (cf. \cite{prashanth2016mlj}). Next, DRM is an expectation w.r.t. a distorted distribution that is arrived at using a distortion function that alters the underlying cumulative distribution function (CDF). Popular risk measures like VaR and CVaR can be seen as special cases of DRM using appropriate distortion functions. However, VaR is not a popular objective for risk-sensitive optimization since it is not coherent\footnote{A risk measure is said to be coherent if it is translation invariant, sub-additive, positive homogeneous, and monotonic \cite{artzner99}.}, while CVaR, though coherent, is not preferable, as it considers all rewards below VaR equally, while ignoring all those beyond VaR. A DRM is preferable as it prioritizes all rewards appropriately, rather than assigning equal weight or selectively focusing on a fraction using a tail-based risk measure like CVaR.

We employ the policy gradient approach for solving a risk-sensitive Markov decision process (MDP), with an SRM as the objective. The goal in our formulation is to find a policy that maximizes the SRM of the cumulative reward in an episodic MDP. We propose a template policy gradient algorithm to solve this problem for an abstract SRM. The template algorithm has the following  crucial components: a risk estimation scheme and a gradient estimation scheme. The risk estimation scheme for an abstract SRM is assumed to guarantee a $O(\nicefrac{1}{m})$ mean-square error (MSE), where $m$ is the number of episodes. With an expected value objective in a risk-neutral setting, this MSE requirement is natural. For the case of MVRM and DRMs, we manifest such a bound for natural estimators. We would like to add that, unlike expected value where a sample mean was a good estimator, estimating a DRM is more challenging since the episodes are obtained using the CDF of the cumulative reward, while DRM is an expectation with a distorted distribution implying an estimate of the underlying CDF is necessary, or a sample mean is not sufficient for DRM estimation.

For the purpose of gradient estimation, we employ the smoothed functional (SF) approach. This scheme falls under the realm of simultaneous perturbation methods \cite{shalabh_book}, which estimate the gradient of a function given noisy observations. Simultaneous perturbation methods in general, and SF methods in particular, are efficient and easy to implement as they require only two function measurements for estimating the gradient, irrespective of the parameter dimension. The choice of the SF scheme for estimating the gradient of an abstract SRM is not arbitrary. For some risk measures, it is not possible to employ the likelihood ratio method to arrive at a policy gradient theorem. This is true for the mean variance risk measure, as shown in \cite{prashanth2016mlj}. This is also unlike the classic expected value objective, for which one could use the policy gradient theorem to arrive at a gradient estimation scheme based on the likelihood ratio method.

We now summarize our contributions. First, we propose two template policy gradient algorithms with an SRM as the objective. The first algorithm operates in an on-policy RL setting, while the second caters to the off-policy RL setting. Second, we derive non-asymptotic bounds that quantify the rate of convergence of our proposed algorithms to a stationary point of an SRM. As special cases, we establish that our algorithms and associated theoretical guarantees apply to optimization of mean-variance and distortion risk measures, respectively, in a risk-sensitive RL context. To the best of our knowledge, policy gradient algorithm with non-asymptotic convergence guarantees are not available in the literature for SRMs in general, and for the special cases of mean-variance risk measure and DRMs in particular. Our non-asymptotic bound for the template algorithm can be used as a blackbox to characterize the convergence rate for SRMs beyond mean-variance and DRM. In particular, one can arrive at a $O(1/\epsilon^{2})$ bound on the number of iterations for convergence to an $\epsilon$-stationary point of the SRM, provided one verifies the necessary assumptions that guarantee smoothness of SRM and a MSE bound on the SRM estimators.

\textbf{Related work.}
In \cite{tamar2015}, the authors propose a policy gradient algorithm for an abstract coherent risk measure, and derive a policy gradient theorem using the dual representation of a coherent risk measure. Their estimation scheme requires solving a convex optimization problem. Also, they establish asymptotic consistency of their proposed gradient estimate. In contrast, our estimation scheme is computationally inexpensive, and our theoretical guarantees are non-asymptotic in nature.
In \cite{prashla2021}, the authors survey policy gradient algorithms for optimizing different risk measures in a constrained as well as an unconstrained RL setting.
They provide a non-asymptotic bound of $O(1/N^{1/3})$ for an abstract smooth risk measure, assuming a gradient oracle that satisfies certain bias-variance conditions. In contrast, we provide concrete gradient estimation schemes in a risk-sensitive RL setting, and more importantly, we derive an improved non-asymptotic bound of order $O(1/\sqrt{N})$.
In \cite{prashla16} the authors consider a CPT-based objective in an RL setting, and they employ simultaneous perturbation stochastic approximation (SPSA) method for the gradient estimation, and provide asymptotic convergence guarantees for their algorithm. The optimization of a DRM is closely related to that of CPT. Under general conditions on the policy parameterization, which are usually employed in the analysis of policy gradient algorithms, we show that DRM is smooth, in turn leading a non-asymptotic bound of $O(1/\sqrt{N})$. This is unlike \cite{prashla16}, where the authors provide asymptotic guarantees assuming the policy parameterization ensures that the CPT-value is three times continuously differentiable --- a condition that is hard to verify in practice.
In a non-RL context, the authors in \cite{glynn21} study the sensitivity of DRM using an estimator that is based on the generalized likelihood ratio method, and establish a central limit theorem for their gradient estimator. In \cite{holland22}, the authors analyze the optimization of spectral risk measures in an empirical risk minimization framework that assumes convex losses.

This paper is an extended version of an earlier work (see \cite{nv23}). Although the order of the convergence bounds remains the same, this version corrects errors from the earlier work by modifying an assumption and revising the proof.

The rest of the paper is organized as follows: Section \ref{sec:prelims} provides the preliminaries for a risk-sensitive episodic problem. Section \ref{sec:sf} introduces our proposed policy gradient template for smooth risk measures. Section \ref{sec:main} presents the non-asymptotic bounds for our proposed algorithms. Section \ref{sec:app} outlines the application of our algorithms to two prominent examples of SRM, namely, DRM and MVRM. Finally, Section \ref{sec:conclusions} provides the concluding remarks.
\section{Preliminaries}
\label{sec:prelims}
We consider an MDP with a state space $\mathscr{S}$ and an action space $\mathscr{A}$. We assume that $\mathscr{S}$ and $\mathscr{A}$ are finite spaces. Let $r:\mathscr{S}\times\mathscr{A}\times\mathscr{S}\to [-r_{\textrm{max}},r_{\textrm{max}}], r_{\textrm{max}}\in\mathbb{R}^{+}$ be the single stage scalar reward, and $p:\mathscr{S}\times\mathscr{S}\times\mathscr{A} \to [0,1]$ be the transition probability function. The actions are selected using parameterized stochastic policies $\{\pi_\theta:\mathscr{S}\times\mathscr{A}\times\mathbb{R}^d\to[0,1],\theta\in\mathbb{R}^d\}$.

We consider episodic problems, where each episode starts at a fixed state $S_0$, and terminates at a special absorbing state $0$. We denote by $S_t$ and $A_t$, the state and action at time $t\in\{0,1,\cdots\}$ respectively. We consider an episodic MDP (cf. Chapter 3, \cite{sutton_book}), where the agent interacts with the environment in discrete episodes, each comprising a finite number of steps.
Ensuring a finite number of steps necessitates the establishment of termination conditions to indicate the conclusion of an episode within the framework of the MDP design. These conditions involve reaching a designated terminal state $0$, or reaching the maximum allowable number of steps per episode. Whether the maximum episode length is attained or not, the episode concludes upon reaching the terminal state, and vice versa.

Incorporating a maximum allowable episode length into the MDP design ensures that episodes conclude within a finite number of steps, while still allowing for variable episode lengths. This flexibility arises from the possibility of reaching the terminal state before reaching the maximum allowable number of steps. Users have the flexibility to set the maximum allowable number of steps according to the problem at hand, a practice frequently observed in OpenAI Gym environments \cite{openai}. Establishing a maximum episode length in real-time scenarios yields various benefits, facilitating the system's learning or decision making within a reasonable timeframe and thereby enhancing overall efficiency in real-time applications. Moreover, it assists in ensuring timely task completion by preventing the system from getting stuck in unproductive states or decision loops.

Let $T$ denote the random length of an episode. Within the MDP framework, we define the maximum episode length, tailored to the problem at hand, denoted as $M_e>0$, ensuring that $T \leq M_e$ a.s., i.e.,
    \begin{align}
        \label{eq:M_e}
        \exists M_e >0 : T \leq M_e < \infty \textrm{ a.s}.
    \end{align}

We make the following assumption on the parameterized policies $\{\pi_\theta,\theta\in\mathbb{R}^d\}$:
\begin{assumption}
    \label{as:nabla_logpi}
    $\exists M_{d},M_{h}>0: \forall \theta\in \R^d, \forall a\in \mathscr{A}, s \in \mathscr{S}$, $\left\lVert\nabla\log \pi_{\theta}(a\mid s)\right\lVert\leq M_d$, and $\left\lVert\nabla^2\log \pi_{\theta}(a\mid s)\right\lVert\leq M_h$,
    where $\lVert\cdot\rVert$ is the $d$-dimensional Euclidean norm when the operand is a vector, and the operator norm when the operand is a matrix.
\end{assumption}

An assumption like \ref{as:nabla_logpi} is common for analyzing policy gradient algorithms (cf. \cite{zhangK2020,papini2018}). To illustrate the plausibility of \ref{as:nabla_logpi}, let us examine a policy that follows Gibbs distribution, i.e., $\pi_\theta(a|s)=\nicefrac{exp(h(s,a,\theta))}{\sum_{b\in\mathscr{A}}exp(h(s,b,\theta))}$, where $h:\mathscr{S}\times \mathscr{A}\times\R^d\to\R$ is a user defined function.
We can see that,
\begin{align*}
    &\nabla \log \pi_\theta(a|s)=\nabla h(s,a,\theta) - \sum_{b\in\mathscr{A}}\pi_\theta(b|s)\nabla h(s,b,\theta);\\
    &\nabla^2 \log \pi_\theta(a|s)=\nabla^2 h(s,a,\theta)\\
    & + \left(\sum_{b\in\mathscr{A}}\pi_\theta(b|s) \nabla h(s,b,\theta)\right)\!\left(\sum_{b\in\mathscr{A}}\pi_\theta(b|s) \nabla h(s,b,\theta)\right)^\top\\
    &-\sum_{b\in\mathscr{A}}\pi_\theta(b|s)\left(\nabla^2 h(s,b,\theta)+\nabla h(s,b,\theta)\nabla h(s,b,\theta)^\top\right).
\end{align*}
If we choose linear policy class, i.e., $h(s,a,\theta)=\phi(s,a)^\top\theta$, with bounded features, i.e., $\left\lVert\phi(s,a)\right\rVert \leq M$, then
$\left\lVert \nabla \log \pi_\theta(a|s)\right\rVert\leq\lvert\mathscr{A}\rvert M $, and
$\left\lVert\nabla^2 \log \pi_\theta(a|s)\right\rVert \leq   \lvert \mathscr{A}\rvert M^2+\lvert \mathscr{A}\rvert^2  M^2$. Since we consider finite state-action spaces, it is easy to arrive at constants $M_d$ and $M_h$ that ensure \ref{as:nabla_logpi} holds.

We denote by $S_t$ and $A_t$, the state and the action at time $t\in\{0,1,\cdots\}$ respectively. The cumulative discounted reward $R^\theta$, which is a random variable, is defined as follows:
\begin{align}
    R^\theta=\sum\limits_{t=0}^{T-1}\gamma^t r(S_t,A_t,S_{t+1}),  \forall \theta \in \R^d,
    \label{eq:Rtheta}
\end{align}
where $A_t \sim \pi_\theta(\cdot, S_t)$, $S_{t+1}\sim p(\cdot,S_t,A_t)$, $\gamma \in (0,1)$ is the discount factor, and $T$ is the random length of an episode. Note that, $\forall \theta \in \mathbb{R}^d, \lvert R^\theta \rvert \leq r_{\textrm{max}}(\nicefrac{1-\gamma^{M_e}}{1-\gamma})=M_r$ a.s.

On-policy learning is a scheme where a policy parameter $\theta$ is optimized using the data collected by the same policy $\pi_\theta$. In contrast, off-policy learning is a scheme where we optimize $\theta$ using data collected by a different behavior policy $b$.

In an off-policy setting, we collect episodes from $b$ and estimate the values of $\pi_\theta$, using importance sampling ratios.
We assume that the target policy $\pi_\theta$ is absolutely continuous w.r.t. the behavior policy $b$, i.e.,
\begin{assumption}
    \label{as:b_pol}
    $\forall \theta \in \!\R^d, b(a | s) \!=\!0 \Rightarrow \pi_\theta(a | s)\!=\!0,\forall a \in \mathscr{A}, \forall s \in \mathscr{S}$.
\end{assumption}
Assumption \ref{as:b_pol} is standard in an off-policy RL setting (cf. \cite{sutton08}).

The cumulative discounted reward $R^b$, which is a random variable, is defined as follows:
\begin{align}
    \label{eq:Rb}
    R^b=\sum\limits_{t=0}^{T-1}\gamma^t r(S_t,A_t,S_{t+1}),
\end{align}
where $A_t \sim b(\cdot, S_t)$, $S_{t+1}\sim p(\cdot,S_t,A_t)$, $\gamma \in (0,1)$, and $T$ is the random length of an episode.

The importance sampling ratio $\psi^\theta$ is defined by
\begin{align}
    \label{eq:psi}
    \psi^\theta = \prod\limits_{t=0}^{T-1}\frac{\pi_{\theta}(A_t\mid S_t)}{b(A_t\mid S_t)}.
\end{align}
From \ref{as:nabla_logpi} and \ref{as:b_pol}, we obtain $\forall \theta \in \R^d,\pi_{\theta}(a|s)>0$ and $b(a|s) >0$, $\forall a \in \mathscr{A}, \textrm{ and } \forall s \in \mathscr{S}$. This fact in conjunction with \eqref{eq:M_e} implies the following bound for $\psi^\theta$:
\begin{align}
    \label{eq:is_ratio}
    \exists M_s>0 : \forall \theta\in\R^d, \psi^\theta \leq M_s, \textrm{ a.s}.
\end{align}

The cumulative discounted reward is a random variable as there is randomness in state transition as modeled by the transition probability function as well as in the action selection in the case of stochastic policies. We consider a smooth risk measure $\rho$ as an objective function, which provides a numerical value that represents certain aspects of this random variable.
\begin{definition}
    A risk measure is smooth if it satisfies the following condition: There exists a positive constant $L_{\rho'}$ such that,
    \begin{align}
       \forall \theta_1, \theta_2 \in \R^d, \left\lVert \nabla\rho(\theta_1)-\nabla\rho(\theta_2)\right\rVert \leq L_{\rho'} \left\lVert \theta_1 \!-\! \theta_2 \right\rVert.
    \end{align}
\end{definition}
Under relatively general conditions, DRM and MVRM can be considered as instances of smooth risk measures. We establish this fact in Section \ref{sec:app}.

Our goal is to find a policy parameter $\theta^*$ that maximizes the objective function $\rho$, i.e,
\begin{align}
    \label{eq:max_theta}
    \theta^*\in\textrm{arg}\!\max_{\theta \in \mathbb{R}^d}  \rho(\theta).
\end{align}
\section{Policy gradient template}
\label{sec:sf}
We propose two policy gradient algorithms for optimizing a smooth risk measure. The first algorithm operates in an on-policy RL setting, and Algorithm \ref{alg:onP} presents the pseudocode. The second algorithm caters to an off-policy RL setting, with a pseudocode that follows the template in Algorithm \ref{alg:onP} with variations in estimation.
There are two crucial ingredients in each of these policy gradient algorithms:
\begin{enumerate}
    \item Risk estimation: This refers to the problem of estimating the value of a smooth risk measure for a given policy parameter, say $\theta$. In an on-policy setting, the estimation scheme has access to a mini-batch of episodes  from the policy $\pi_\theta$ itself. On the other hand, in an off-policy setting, the estimation scheme has to use the episodes simulated using a behavior policy.
    \item Gradient estimation: This refers to the estimation of the policy gradient $\nabla \rho(\theta)$ for a given parameter $\theta$. Such an estimate would be used to perform stochastic gradient ascent in the policy parameter.
\end{enumerate}
The estimation scheme is specific to the risk measure considered. For the theoretical guarantees in the next section, we require the following bound on the estimate $\hat\rho_m(\theta)$ of the risk $\rho(\theta)$, given $m$ episodes: For some positive constant $C_1$,
\begin{align}
    \label{eq:msebound}
    \E\left[\left\lvert \hat{\rho}_m(\theta) - \rho(\theta) \right\rvert^2\right]\leq\frac{C_1}{m}.
\end{align}
The condition above relates to the mean-square error of the risk estimator, and the rate of $O(\nicefrac{1}{m})$ is natural, considering such a bound is reasonable even for the case of an expected value objective.
For the two applications with mean-variance and distortion risk measures, we shall establish later that the estimators of these risk measures satisfy the condition specified above.

For handling the problem of gradient estimation, both algorithms use an SF-based estimation scheme. The choice of this gradient estimation scheme is not arbitrary. The application of the likelihood ratio method to derive a policy gradient theorem is not viable for certain risk measures. This limitation is evident in the case of the mean-variance risk measure, as demonstrated in \cite[Lemma~1]{prashanth2016mlj}. Specifically, when considering the policy gradient expression for the squared value $\E\left[(R^\theta)^2\right]$, it incorporates the gradient of the value function at each state of the MDP. Consequently, this inclusion presents challenges in accurately estimating the gradient.
In the aforementioned reference, the authors employed SPSA, a popular simultaneous perturbation method to workaround the policy gradient expression. In our work, we use SF, which also falls under the realm of simultaneous perturbation methods for gradient estimation. Moreover, unlike \cite{prashanth2016mlj}, we consider a broad class of smooth risk measures, and more importantly, we establish non-asymptotic bounds that quantify the rate of convergence of our proposed SF-based policy gradient algorithms.

The SF-based gradient estimation is a zeroth-order gradient estimation scheme, where the gradient is estimated from perturbed function values (cf. \cite{nesterov2017,shalabh_book, shamir}). The SF method forms a smoothed version of the objective function $\rho(\cdot)$ as $\rho_{\mu}(\cdot)$ and uses the gradient $\nabla \rho_{\mu}$ as an approximation for $\nabla \rho$. The smoothed functional $\rho_{\mu}(\theta)$ is defined as
\begin{align}
    \label{eq:rho_mu}
    \rho_{\mu}(\theta) = \mathbb{E}_{u \in \mathbb{B}^d}\left[\rho({\theta+\mu u})\right],
\end{align}
where $u$ is sampled uniformly at random from the unit ball $\mathbb{B}^d=\{x\in\mathbb{R}^d \mid \lVert x \rVert \leq 1\}$, and $\mu \in (0,1]$ is the smoothing parameter.
From \cite[Lemma 2.1]{flaxman}, we obtain the following expression for the gradient of $\rho_{\mu}(\theta)$.
\begin{align}
    \label{eq:del_rho_mu}
    \nabla\rho_{\mu}(\theta)=\mathbb{E}_{v\in\mathbb{S}^{d-1}}\left[\frac{d}{\mu}\rho({\theta+\mu v})v\right],
\end{align}
where $v$ is sampled uniformly at random from the unit sphere $\mathbb{S}^{d-1}=\{x\in\mathbb{R}^d \mid \lVert x \rVert = 1\}$.
In a deterministic optimization setting with  perfect measurements of $\rho(\cdot)$, the gradient $\nabla\rho_{\mu}(\theta)$ is estimated as follows:
\begin{align}
    \label{eq:hat_nabla_rho_0}
    \widehat{\nabla}_{\mu,n}\rho(\theta) = \frac{d}{n}\sum\limits_{i=1}^{n} \frac{\rho({\theta+\mu v_i}) - \rho({\theta - \mu v_i})}{2\mu}v_i,
\end{align}
where $\forall i, v_i$ is sampled uniformly at random from $\mathbb{S}^{d-1}$. The gradient estimate is averaged over $n$ unit vectors to reduce the variance.
Using the proof technique from \cite{nv1}, we show that $\widehat{\nabla}_{\mu,n}\rho(\theta)$ is an unbiased estimator of $\nabla\rho_{\mu}(\theta)$, see Appendix~A for the details.

In a typical RL setting, we may not have direct measurements of $\rho(\cdot)$, which need to be estimated using sample episodes. Let  $\hat{\rho}_m(\cdot)$ be the estimator for $\rho(\cdot)$, then we use a gradient estimator as given below:
\begin{align}
    \label{eq:hat_nabla_hat_rho}
    \widehat{\nabla}_{\mu,n}\hat{\rho}_m(\theta) = \frac{d}{n}\sum\limits_{i=1}^{n} \frac{\hat{\rho}_m({\theta+\mu v_i}) - \hat{\rho}_m({\theta - \mu v_i})}{2\mu}v_i.
\end{align}
We solve \eqref{eq:max_theta} using the following update iteration:
\begin{align}
    \label{eq:approx_theta_update}
    \theta_{k+1} = \theta_k + \alpha \widehat{\nabla}_{\mu,n} \hat{\rho}_m({\theta_k}),
\end{align}
where $\theta_0$ is set arbitrarily, and $\alpha$ is the step-size.

We consider two algorithms, both armed with a risk estimator $\hat{\rho}_m(\cdot)$ and a risk gradient estimate using SF. In our first algorithm OnP-SF, $\hat{\rho}_m(\cdot)$ uses an on-policy evaluation.  Algorithm \ref{alg:onP} presents the pseudocode of OnP-SF.
\begin{algorithm}[ht]
    \caption{OnP-SF}
    \label{alg:onP}
    \begin{algorithmic}[1]
        \STATE \textbf{Input}: Parameterized form of the policy $\pi$, iteration limit $N$, step-size $\alpha$, perturbation parameter $\mu$, and batch sizes $m$ and $n$;
        \STATE \textbf{Initialize}: Target policy parameter $\theta_{0} \in \mathbb{R}^d$, and the discount factor $\gamma \in (0,1)$;
        \FOR {$k=0,\hdots, N-1$ }
        \FOR {$i=1,\hdots, n$ }
        \STATE Get $[v_i^1, \hdots, v_i^d] \in \mathbb{S}^{d-1}$;
        \STATE Generate $m$ episodes each using $\pi_{(\theta_k \pm \mu v_i)}$;
        \STATE Estimate $\hat{\rho}_m({\theta_k \pm \mu v_i)}$;
        \ENDFOR
        \STATE Use \eqref{eq:hat_nabla_hat_rho} to estimate $\widehat{\nabla}_{\mu,n} \hat{\rho}_m({\theta_k})$;
        \STATE Use \eqref{eq:approx_theta_update} to calculate $\theta_{k+1}$;
        \ENDFOR
        \STATE \textbf{Output}: Policy $\theta_R$, where $R \sim \mathcal{U}\{0,N-1\}$
    \end{algorithmic}
\end{algorithm}

Each iteration of OnP-SF requires $2mn$ episodes corresponding to $2n$ perturbed policies. In some practical applications, it may not be feasible to generate system trajectories corresponding to different perturbed policies. In our second algorithm OffP-SF, we overcome the aforementioned problem by performing the off-policy evaluation. Using the off-policy setting, the number of episodes needed in each iteration of our algorithm can be reduced to $m$.
The pseudocode of OffP-SF is similar to Algorithm \ref{alg:onP} with the following deviations: The estimate $\hat{\rho}_m({\theta_k \pm \mu v_i)}$ in step 7 is performed in a off-policy fashion, and for this purpose $m$ episodes are generated only once using the behavior policy. In contrast, step 6 in Algorithm \ref{alg:onP} requires simulation of $m$ episodes in each iteration using the current policy parameter $\pi_{\theta_k \pm \mu v_i}$.
\section{Main results}
\label{sec:main}
Our non-asymptotic analysis establishes a bound on the number of iterations of our proposed algorithms to find an $\epsilon$-stationary point of the smooth risk measure, which is defined below.
\begin{definition}[$\epsilon$-stationary point]
    \label{def:esolution}
    Fix $\epsilon>0$. Let $ \theta_R$ be the random output of an algorithm. Then, $\theta_R$ is called an $\epsilon$-stationary point of problem \eqref{eq:max_theta}, if $\mathbb{E}\left[ \left\Vert \nabla \rho \left( \theta_R \right) \right\rVert^2\right] \leq \epsilon$, where the expectation is over $R$.
\end{definition}
For a non-convex objective function, it is common in optimization literature to establish a convergence rate result to an $\epsilon$-stationary point. Such a convergence notion is used in the analysis of policy gradient algorithms as well, cf.  \cite{papini2018,shen2019hessian,zhangK2020}.

\subsection{Bounds for OnP-SF/OffP-SF}
We  make the following assumptions for the sake of analysis.
\begin{assumption}
    \label{as:mse}
    $\forall \theta \in \R^d$, $\E\left[\left\lvert \hat{\rho}_m(\theta) - \rho(\theta) \right\rvert^2\right]\leq\frac{C_1}{m}$.
\end{assumption}
\begin{assumption}
    \label{as:lip}
    $\forall \theta_1, \theta_2 \in \R^d$, $\left\lvert \rho(\theta_1)-\rho(\theta_2)\right\rvert \leq L_{\rho} \left\lVert \theta_1 - \theta_2 \right\rVert$.
\end{assumption}
\begin{assumption}
    \label{as:smooth}
    $\forall \theta_1, \theta_2 \in \R^d$, $\left\lVert \nabla \rho (\theta_1)-\nabla \rho (\theta_2)\right\rVert \leq L_{\rho'} \left\lVert \theta_1 \!-\! \theta_2 \right\rVert$.
\end{assumption}

We present bounds for an iterate $\theta_R$ that is chosen uniformly at random from $\{\theta_0,\cdots,\theta_{N-1}\}$. The bound that we present below applies to the template algorithm for on-policy as well as off-policy RL settings. Moreover, the bound below is for a general step-size, smoothing parameter and batch size parameters. Subsequently, we specialize this result to arrive at a $O(\nicefrac{1}{\sqrt{N}})$ bound on $\mathbb{E}\left[ \left\Vert \nabla \rho \left( \theta_R \right) \right\rVert^2\right]$. The proofs can be found in Appendix~A. 
\begin{proposition}(OnP-SF/OffP-SF)
    \label{pr:non_asym_sf}
    Assume \ref{as:mse}-\ref{as:smooth}. Let $\{\theta_i,i=0,\cdots,N-1\}$ be the policy parameters generated by OnP-SF/OffP-SF, and let $\theta_R$ be chosen uniformly at random from this set. Let $\rho^*=\max_{\theta\in\R^d}\rho(\theta)$. Then
    \begin {align}
    \label{eq:tm_sf}
    &\mathbb{E}\left[\left\lVert \nabla \rho(\theta_R)\right\rVert^2\right]\nonumber\\
    &\quad\leq \frac{2 \left(\rho^* - \rho(\theta_{0})\right)}{N \alpha}
    + L_{\rho'} \alpha \left (\frac{3d^2C_1}{2\mu^2m} + 3d^2L_{\rho}^2\right )\nonumber\\
      &\quad  + \left( \frac{3C_1d^2}{\mu^2m}+\frac{48e^2d^2L_{\rho}^2}{n}+\frac{3\mu^2 d^2 L_{\rho'}^2}{4} \right).
    \end {align}
    where $L_\rho,L_\rho'$, and $C_1$ are as in \ref{as:mse}-\ref{as:smooth}.
\end{proposition}
A straightforward specialization of the bound in \eqref{eq:tm_sf} with specific choices for the step-size $\alpha$, smoothing parameter $\mu$, and batch sizes $m$ and $n$ leads to following bound for OnP-SF and OffP-SF algorithms.
\begin{theorem}(\textbf{OnP-SF/OffP-SF})
    \label{tm:sf}
    Set $\alpha=\frac{1}{\sqrt{N}}$, $\mu=\frac{1}{\sqrt[4]{N}}$, $n=\sqrt{N}$, and $m=N$. Then, under the conditions of Proposition \ref{pr:non_asym_sf}, we have
    \begin {align*}
    &\mathbb{E}\left[\left\lVert \nabla \rho(\theta_R)\right\rVert^2\right]\nonumber\\
    &\quad\leq \frac{2 \left(\rho^* - \rho(\theta_{0})\right)}{\sqrt{N}}
    + \frac{3d^2L_{\rho'} C_1}{2N} \\
    &\quad+ \frac{3d^2 \left(L_{\rho}^2L_{\rho'}+C_1+16e^2L_{\rho}^2\right )}{\sqrt{N}}+\frac{3 d^2 L_{\rho'}^2}{4\sqrt{N}} .
    \end {align*}
\end{theorem}
\begin{remark}
    The results above show that after $N$ iterations of \eqref{eq:approx_theta_update}, OnP-SF/OffP-SF return an iterate that satisfies $\E\!\left[\left\lVert\nabla\rho(\theta_R)\right\rVert^2\right]=O\left(\nicefrac{1}{\sqrt{N}}\right)$. To put it differently, to find an $\epsilon$-stationary point of the smooth risk measure objective, an order $O(\nicefrac{1}{\epsilon^2})$ iterations of OnP-SF/OffP-SF are enough.
\end{remark}
\begin{remark}
        The bounds obtained for OnP-SF and OffP-SF are $O(\nicefrac{1}{\sqrt{N}})$. The OnP-SF requires $\Theta(N\sqrt{N})$ episodes to achieve the aforementioned rate.
        On the other hand,  OffP-SF requires only $\Theta(N)$ episodes. This difference arises because, in the off-policy setting with $m$ episodes from policy $b$, one can achieve $n$ averaging in the gradient estimate by calculating $mn$ episodes corresponding to policy $\pi$. This approach does not require simulating additional episodes.
\end{remark}
\begin{remark}
    Typical results in risk-sensitive RL literature are produces asymptotic in nature. An exception is a result from \cite{prashla2021}, where non-asymptotic bounds of $O\left(\nicefrac{1}{N^{1/3}}\right)$ are presented. In contrast, we derive $O(\nicefrac{1}{\sqrt{N}})$ bounds for SRMs.
\end{remark}
\section{Applications}
\label{sec:app}
Under relatively general conditions, DRM and MVRM can be considered as instances of smooth risk measures. We describe these risk measures in the following sections.
\subsection{Distortion risk measures (DRM)}
\label{subsec:drm}
The DRM of  $R^\theta$, defined in \eqref{eq:Rtheta} is the expected value of $R^\theta$ under a distortion of the CDF  $F_{R^{\theta}}$, attained using a given distortion function $g(\cdot)$.  We denote by $\rho_g(\theta)$ the DRM of $R^\theta$, and is defined as follows:
\begin{align}
    \label{eq:rho_g_1}
    \rho_g(\theta)\!=\!\!\int_{-M_r}^{0}\!\!\!\!(g(1\!-\!F_{R^{\theta}}(x))\!-\!1) dx + \!\!\int_{0}^{M_r}\!\!\!\! g(1\!-\!F_{R^{\theta}}(x))dx.
\end{align}
The distortion function $g:[0,1]\to[0,1]$ is non-decreasing, with $g(0)=0$ and $g(1)=1$. We can see that $\rho_g(\theta)\!=\!\E[R^\theta]$, if $g(\cdot)$ is the identity function. The limit of the integration in \ref{eq:rho_g_1}, $M_r = r_{\textrm{max}}(\nicefrac{1-\gamma^{M_e}}{1-\gamma})$.
A few examples of $g(\cdot)$ are available in Table \ref{tb:g} and their plots are in Figure \ref{fg:g}.
\begingroup
\renewcommand{\arraystretch}{1.25}
\begin{table}[t]
    \caption{Examples of distortion functions}
    \label{tb:g}
    \begin{center}
        \begin{small}
             \begin{tabular}{ll}
                \toprule
                Dual-power function  & $g(s)=1-(1-s)^{\lambda}$,\; ${\lambda}\geq 2$\\
                Quadratic function    & $g(s)=(1+{\lambda}) s- {\lambda} s^2$, \;$0\leq {\lambda}\leq 1$\\
                Exponential function &$g(s)=\nicefrac{1-\exp(-{\lambda} s)}{1-\exp(-{\lambda} )}$,\;${\lambda}>0$\\
                Square-root function  & $g(s)=\nicefrac{\sqrt{1+{\lambda} s}-1}{\sqrt{1+{\lambda} }-1}$,\; ${\lambda}>0$\\
                Logarithmic function  & $g(s)=\nicefrac{\log(1+{\lambda} s)}{\log(1+{\lambda} )}$,\; ${\lambda}>0$\\
                \bottomrule
            \end{tabular}
        \end{small}
    \end{center}
\end{table}
\begin{figure}[t]
    \begin{center}
        {\includegraphics[width=0.8\columnwidth]{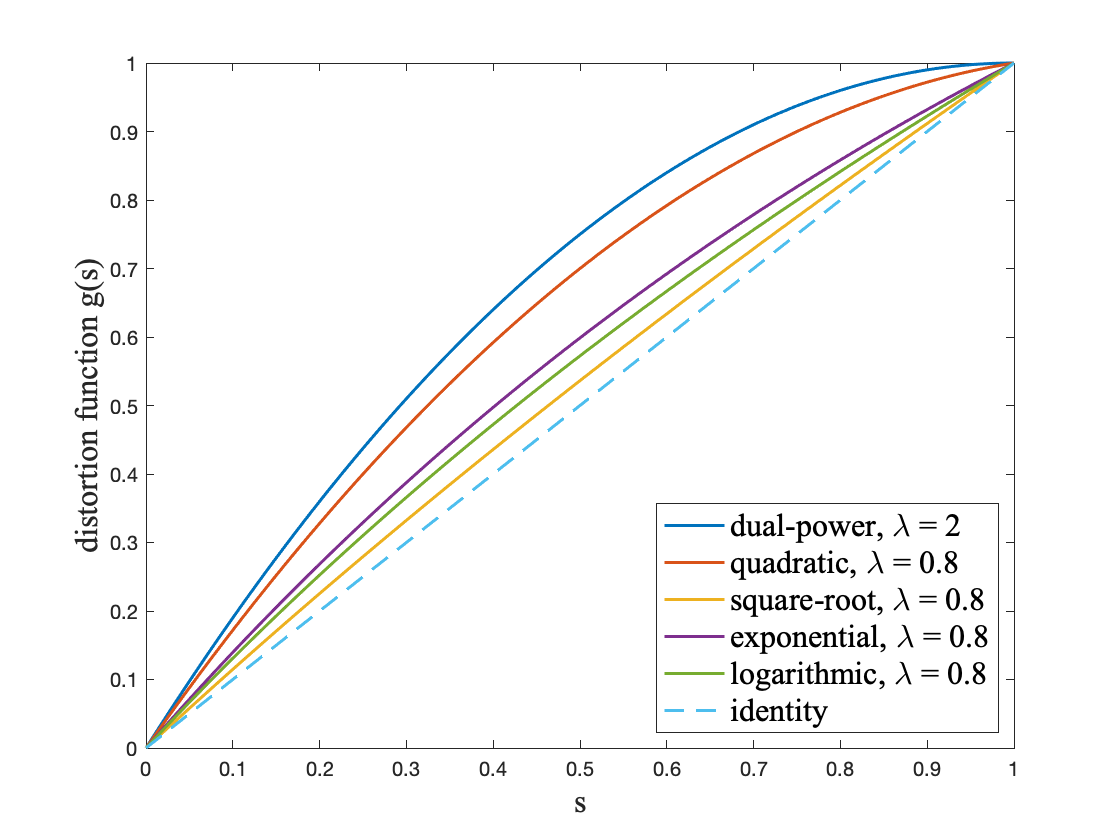}}
        \caption{Examples of distortion functions}
        \label{fg:g}
    \end{center}
\end{figure}

Recall that the optimization problem in \eqref{eq:max_theta} is solved using stochastic gradient algorithm, and for each update iteration, we require estimates of $\rho_g(\cdot)$. In the following sections, we describe our algorithms that estimate DRM in on-policy and off-policy RL settings, respectively.
\subsubsection{On-policy DRM estimation}
\label{subsubsec:onpolicy}
We generate $m$ episodes using the policy $\pi_\theta$, and estimate the CDF $F_{R^{\theta}}(\cdot)$ using sample averages.
We denote by $R^{\theta}_i$ the cumulative reward of the episode $i$.
We form the estimate $G^m_{R^{\theta}}(\cdot)$ of $F_{R^{\theta}}(\cdot)$ as follows:
\begin{align}
    \label{eq:G}
    G^m_{R^{\theta}}(x) = \frac{1}{m}\sum\limits_{i=1}^m \1\{R^{\theta}_i\leq x\}.
\end{align}
Now, we form an estimate $\hat{\rho}_g^G(\theta)$ of $\rho_g(\theta)$ as follows:
\begin{align}
    \label{eq:hat_rho_G}
    \hat{\rho}_g^G(\theta)\!=\!\!\!\int_{-M_r}^{0}\!\!\!\!\!(g(1\!-\!G^m_{R^{\theta}}(x))\!-\!1) dx +\!\! \int_{0}^{M_r}\!\!\!\!\! g(1\!-\!G^m_{R^{\theta}}(x))dx.
\end{align}
Comparing \eqref{eq:hat_rho_G} with \eqref{eq:rho_g_1}, it is apparent that we have used the empirical distribution function $G^m_{R^{\theta}}$ in place of the true CDF $F_{R^{\theta}}$.

We simplify \eqref{eq:hat_rho_G} in terms of order statistics as follows:
\begin{align}
    \label{eq:hat_rho_G1}
    \hat{\rho}_g^G(\theta)=  \sum\limits_{i=1}^{m} {R^\theta_{(i)}} \left(g\left(1\!-\! \frac{i\!-\!1}{m}\right) \!-\! g\left(1\!-\! \frac{i}{m}\right)\right),
\end{align}
where $R^\theta_{(i)}$ is the $i^{th}$ smallest order statistic of the samples $\{R^\theta_1,\cdots R^\theta_m\}$. The reader is referred to  Lemma~13 in Appendix~B for the proof. If we choose the distortion function as the identity function, then the estimator in \eqref{eq:hat_rho_G1} is merely the sample mean.

We make the following assumptions to ensure the Lipschitzness, and smoothness of the DRM $\rho_g$.
\begin{assumption}
    \label{as:g'_bound}
    $\exists M_{g'},M_{g''}>0: \forall t\in(0,1)$, $\lvert g'(t)\rvert \leq M_{g'}$, and $ \lvert g''(t) \rvert \leq M_{g''}$.
\end{assumption}
The assumption \ref{as:g'_bound} helps us establish that the distortion functions and its derivative are Lipschitz continuous.
A few examples of distortion functions, which satisfy \ref{as:g'_bound} are given in Table \ref{tb:g}.

A critical requirement for establishing convergence guarantee is a bound on the MSE of the risk estimation scheme, as given in \ref{as:mse}. The result below shows that this MSE requirement is met by the DRM estimator \eqref{eq:hat_rho_G}. The proof can be found in Appendix~B.

\begin{lemma}
    \label{lm:est_error_G}
    Assume \ref{as:nabla_logpi} and \ref{as:g'_bound}. Then,
    \begin{align*}
        \E\left[\left\lvert \rho_g(\theta)- \hat{\rho}_g^G(\theta)\right\rvert^2\right]\leq\frac{16M_r^2M_{g'}^2}{m}.
    \end{align*}
\end{lemma}
\subsubsection{Off-policy DRM estimation}
\label{subsubsec:drm_offpolicy}
We generate $m$ episodes using the policy $b$ to estimate the CDF $F_{R^{\theta}}(\cdot)$ using importance sampling.
We denote by $R^b_i$ the cumulative reward, and $\psi^\theta_i$ the importance sampling ratio of the episode $i$.
We form the estimate $H^m_{R^{\theta}}(\cdot)$ of $F_{R^{\theta}}(\cdot)$ as follows:
\begin{align}
    H^m_{R^{\theta}}(x) &= \min\{\hat{H}^m_{R^{\theta}}(x),1\}, \textrm{ where} \label{eq:H}\\
    \hat{H}^m_{R^{\theta}}(x) &= \frac{1}{m}\sum\limits_{i=1}^m\1\{R^{b}_i\leq x\}\psi^\theta_i.\label{eq:hatH}
\end{align}
In the above, $\hat{H}^m_{R^{\theta}}(x)$ is an empirical estimate of $F_{R^{\theta}}(x)$ as $F_{R^{\theta}}(x) = \E\left[\1\{R^b\leq x\}\psi^\theta  \right]$. Because of the importance sampling ratio, $\hat{H}^m_{R^{\theta}}(x)$ can get a value above $1$. Since we are estimating a CDF, we restrict $\hat{H}^m_{R^{\theta}}(x)$ to $H^m_{R^{\theta}}(x)$.

Now we form an estimate $\hat{\rho}_g^H(\theta)$ of $\rho_g(\theta)$ as
\begin{align}
    \label{eq:hat_rho_H}
    \hat{\rho}_g^H(\theta)\!=\!\!\!\int\nolimits_{-M_r}^{0}\!\!\!\!\!(g(1\!-\!H^m_{R^{\theta}}(x))\!-\!1) dx + \!\!\!\int\nolimits_{0}^{M_r}\!\!\!\!\! g(1\!-\!H^m_{R^{\theta}}(x))dx.
\end{align}

We can simplify \eqref{eq:hat_rho_H} in terms of order statistics as
\begin{align}
    \label{eq:hat_rho_H1}
    \hat{\rho}_g^H(\theta)&=R^b_{(1)}+ \sum_{i=2}^{m} {R^b_{(i)}} g\left(1\!-\! \min\left\{1,\frac{1}{m}\sum_{k=1}^{i-1}\psi^\theta_{(k)}\right\}\right)\nonumber\\
    & - \sum_{i=1}^{m-1}{R^b_{(i)}} g\left(1\!-\! \min\left\{1,\frac{1}{m}\sum_{k=1}^{i}\psi^\theta_{(k)}\right\}\right),
\end{align}
where $R^b_{(i)}$ is the $i^{th}$ smallest order statistic of the samples $\{R^b_1,\cdots R^b_m\}$, and $\psi^\theta_{(i)}$ is the importance sampling ratio of $R^b_{(i)}$. The reader is referred to Lemma~14 in Appendix~B for the proof.

A result in the spirit of Lemma \ref{lm:est_error_G} for the off-policy setting is given below. The result below shows that the MSE requirement in \ref{as:mse} is met by the DRM estimator \eqref{eq:hat_rho_H}. The proof can be found in Appendix~B.

\begin{lemma}
    \label{lm:est_error_H}
    Assume \ref{as:nabla_logpi}-\ref{as:b_pol} and \ref{as:g'_bound}. Then,
    \begin{align*}
        \E\left[\left\lvert \rho_g(\theta)- \hat{\rho}_g^H(\theta)\right\rvert^2\right]\leq\frac{16M_r^2M_{g'}^2M_s^2}{m}.
    \end{align*}
\end{lemma}
\subsubsection{Convergence analysis}
\label{subsubsec:drm_conv}
First we show that the assumptions \ref{as:lip}-\ref{as:smooth} are satisfied for the DRM using the results from the following lemma (see Appendix~B for the proof).
\begin{lemma}
    \label{lm:rho_lip}
    $\forall \theta_1,\theta_2 \in \mathbb{R}^d$,
    \begin{align*}
        &\left\lvert \rho_g(\theta_1)- \rho_g(\theta_2)\right\rvert
        \leq L_\rho \left\lVert \theta_1 \!-\! \theta_2 \right\rVert; L_\rho=2M_rM_{g'}M_eM_d,\\
        & \left\lVert \nabla\rho_g(\theta_1) - \nabla \rho_g(\theta_2) \right\rVert \leq
        L_{\rho'} \left\lVert  \theta_1  - \theta_2\right\rVert;\\
        &\qquad  L_{\rho'}=2M_r M_e \left( M_h M_{g'}+M_eM_d^2 (M_{g'}+ M_{g''})\right).
    \end{align*}
\end{lemma}
The main result that establishes a non-asymptotic bound for Algorithm \ref{alg:onP} with DRM as the risk measure is given below.
\begin{corollary}(DRM-OnP-SF)
    \label{cr:drm_onP}
    Assume \ref{as:nabla_logpi} and \ref{as:g'_bound}. Then, under the conditions of Theorem \ref{tm:sf}, we have
    \begin {align*}
    &\mathbb{E}\left[\left\lVert \nabla \rho_g(\theta_R)\right\rVert^2\right]\\
    &\quad\leq \frac{2 \left(\rho_g^* - \rho_g(\theta_{0})\right)}{\sqrt{N}}
        + \frac{3d^2L_{\rho'} C_1}{2N} \\
        &\quad+ \frac{3d^2 \left(L_{\rho}^2L_{\rho'}+C_1+16e^2L_{\rho}^2\right )}{\sqrt{N}}+\frac{3 d^2 L_{\rho'}^2}{4\sqrt{N}}.
    \end {align*}
    In the above, $\rho_g^*\!=\!\max_{\theta\in\mathbb{R}^d}\rho_g(\theta)$. The constants $C_1, L_{\rho}$, and $L_{\rho'}$ are as in Lemmas \ref{lm:est_error_G} and \ref{lm:rho_lip}, respectively.
\end{corollary}
\begin{proof}
    Lemma \ref{lm:est_error_G} implies \ref{as:mse} holds for DRM estimator.
    Lemma \ref{lm:rho_lip} implies the conditions in \ref{as:lip} and \ref{as:smooth} hold for DRM. The main claim now follows by an application of Theorem \ref{tm:sf}.
\end{proof}
For the off-policy case, a non-asymptotic bound can be inferred from Theorem \ref{tm:sf} in a similar fashion as the on-policy case, with
Lemma \ref{lm:est_error_H} in place of Lemma \ref{lm:est_error_G}.
\begin{corollary}(DRM-OffP-SF)
    \label{cr:drm_offP}
    Assume \ref{as:nabla_logpi}-\ref{as:g'_bound}. Then, under the conditions of Theorem \ref{tm:sf}, we have
    \begin {align*}
    &\mathbb{E}\left[\left\lVert \nabla \rho_g(\theta_R)\right\rVert^2\right]\\
    &\quad\leq \frac{2 \left(\rho_g^* - \rho_g(\theta_{0})\right)}{\sqrt{N}}
        + \frac{3d^2L_{\rho'} C_1}{2N}\\
        &\quad + \frac{3d^2 \left(L_{\rho}^2L_{\rho'}+C_1+16e^2L_{\rho}^2\right )}{\sqrt{N}}+\frac{3 d^2 L_{\rho'}^2}{4\sqrt{N}}.
    \end {align*}
    In the above, $\rho_g^*,L_{\rho}$, and $L_{\rho'}$ are as in Corollary \ref{cr:drm_onP}. The constant $C_1$ is as in Lemma \ref{lm:est_error_H}.
\end{corollary}
\begin{remark}
    If we choose the distortion function as the identity function, then the estimator in \eqref{eq:rho_g_1} is merely  the sample mean, and we recover the guarantees for a risk-neutral policy gradient algorithm. In particular, our bounds match the guarantees given by \cite{nv1}, which employs an SF-based gradient estimation scheme in a risk-neutral setting, and establishes consistency with the bounds of the REINFORCE style policy gradient algorithm.
\end{remark}
\subsection{Mean-variance risk measure (MVRM)}
\label{subsec:mvrm}
The MVRM $\rho_{\lambda}(\theta)$ of $R^\theta$, defined in \eqref{eq:Rtheta},  is given by
\begin{align}
    \label{eq:rho_lambda}
    &\rho_\lambda(\theta)=J(\theta)-\lambda V(\theta), \textrm{ where} \nonumber\\
    &J(\theta)=\E\left[R^\theta\right];\;V(\theta)=\E\left[\left(R^\theta\right)^2\right]-J(\theta)^2.
\end{align}
In the above, $J(\theta)$ is the value function, which is the objective in a risk-neutral RL setting. Further, $V(\theta)$ is the variance of the cumulative reward, and $\lambda$ is a scalar that is used to tradeoff between mean and variance. A popular risk measure in control literature is exponential utility, where the objective is $-\frac{1}{\lambda}\log\mathbb{E}[e^{-\lambda R^\theta}]$, with $R^\theta$ denoting the cumulative reward.
Using a first-order Taylor expansion, it is apparent that
\[-\frac{1}{\lambda}\log\mathbb{E}[e^{-\lambda R^\theta}] = \mathbb{E}[R^\theta] - \frac{\lambda}{2}\text{Var}[R^\theta]+O(\lambda^2).\]
Thus, the MVRM risk measure defined above can be seen as an approximation to the exponential utility risk measure. Optimizing the latter risk measure in an RL context is challenging, and to the best of our knowledge, there is no RL algorithm with a compact parameterization for this problem. Instead of using a parameterized family of policies, the authors in \cite{prashla2021} adopt a different approach by treating the policy as a probability vector over all states and actions. Further, they introduce a two timescale tabular algorithm using Q-values within the context of an average-cost MDP setting (see Section 7.1 of \cite{prashla2021} for the details). In contrast, we present a policy gradient algorithm for MVRM with a provable bound on the rate for stationary convergence.

Next, we describe the estimation of the MVRM in on-policy and off-policy settings, respectively.
\subsubsection{On-policy MVRM estimation}
\label{subsubsec:mvrm_onpolicy}
We generate $m$ episodes using the policy $\pi_\theta$, and estimate $J(\theta)$ and  $V(\theta)$ using sample averages.
We denote by $R^{\theta}_i$ the cumulative reward of the episode $i$. The estimators $\hat{J}_m^{\pi}$ of $J(\theta)$ and $\widehat{V}_m^{\pi}$ of $V(\theta)$ is defined as follows:
\begin{align}
    \label{eq:Jm_pi}
    &\hat{J}_m^{\pi}(\theta)=\frac{1}{m}\sum\limits_{i=1}^{m} R^\theta_i;\\
    \label{eq:Vm_pi}
    &\widehat{V}_m^{\pi}(\theta)=\frac{1}{m-1}\sum\limits_{i=1}^{m} \left(R^\theta_i-\hat{J}_m^{\pi}\right)^2.
\end{align}
Using Theorem 2-3 in \cite[chapter V1]{mood74}, we can see that the above estimates are unbiased.

Using \eqref{eq:Jm_pi} and \eqref{eq:Vm_pi}, we estimate $\rho_\lambda(\theta)$ as follows:
\begin{align}
    \label{eq:hat_rho_lambda_pi}
    &\hat{\rho}_\lambda^{\pi}(\theta) = \hat{J}_m^{\pi}(\theta)-\lambda\widehat{V}_m^{\pi}(\theta).
\end{align}

The result below for the mean-variance estimator \eqref{eq:hat_rho_lambda_pi} satisfies an MSE bound of order $O(1/m)$, in turn verify \ref{as:mse}. The proof can be found in Appendix~C.
\begin{lemma}
    \label{lm:est_error_mvrm_pi}
    Assume \ref{as:nabla_logpi}, and let $m>2$. Then
    \begin{align*}
        \E\left[\left\lvert \hat{\rho}_\lambda^{\pi}(\theta)- \rho_\lambda(\theta) \right\rvert^2\right]
        \leq \frac{8M_r^2+32\lambda^2M_r^4}{m}.
    \end{align*}
\end{lemma}
\subsubsection{Off-policy MVRM estimation}
\label{subsec:mvrm_offpolicy}
We generate $m$ episodes using the policy $b$ to estimate $J(\theta)$ using importance sampling.
We denote by $R^b_i$ the cumulative reward, and $\psi^\theta_i$ the importance sampling ratio of the episode $i$. Since $J(\theta)=\E\left[R^b\psi^\theta\right]$, we estimate it using sample average as follows:
\begin{align}
    \label{eq:Jm_b}
    &\hat{J}_m^{b}(\theta)=\frac{1}{m}\sum\limits_{i=1}^{m} R^b_i \psi^\theta_i;\\
    \label{eq:Vm_b}
    &\widehat{V}_m^{b}(\theta)=\frac{1}{m-1}\sum\limits_{i=1}^{m} \left(R^b_i \psi^\theta_i -\hat{J}_m^{b}\right)^2.
\end{align}
As in the on-policy setting, these estimates are unbiased.

Now using \eqref{eq:Jm_b} and \eqref{eq:Vm_b}, we estimate $\rho_\lambda(\theta)$ as follows:
\begin{align}
    \label{eq:hat_rho_lambda_b}
    &\hat{\rho}_\lambda^{b}(\theta) = \hat{J}_m^{b}(\theta)-\lambda\widehat{V}_m^{b}(\theta).
\end{align}

A result in the spirit of Lemma \ref{lm:est_error_mvrm_pi} for the off-policy setting is given below. The result below for the mean-variance estimator \eqref{eq:hat_rho_lambda_b} satisfies an MSE bound of order $O(1/m)$, in turn verify \ref{as:mse}. The proof can be found in Appendix~C.
\begin{lemma}
    \label{lm:est_error_mvrm_b}
    Assume \ref{as:nabla_logpi}-\ref{as:b_pol}, and let $m>2$. Then
    \begin{align*}
        \E\left[\left\lvert \hat{\rho}_\lambda^{b}(\theta)- \rho_\lambda(\theta) \right\rvert^2\right]
        \leq \frac{8M_r^2M_s^2+32\lambda^2M_r^4M_s^4}{m}.
    \end{align*}
\end{lemma}
\subsubsection{Convergence analysis}
\label{subsubsec:mvrm_conv}
We specialize the result in Proposition \ref{pr:non_asym_sf} to MVRM. Though MVRM is previously analyzed in \cite{tamar2012policy,prashla13}, they only provide asymptotic convergence results.
In the following lemma, we show that the assumptions \ref{as:lip}-\ref{as:smooth} are satisfied for the MVRM. The proof can be found in Appendix~C.
\begin{lemma}
    \label{lm:lip_rho_lambda}
    $\forall \theta_1,\theta_2 \in \R^d$,
    \begin{align*}
        & \left\lvert \rho_\lambda(\theta_1)-\rho_\lambda(\theta_1)\right\rvert
        \leq L_{\rho} \left\lVert \theta_1 - \theta_2 \right\rVert;\\
        &\qquad L_{\rho}=M_rM_eM_d+ 3\lambda M_r^2 M_e M_d, \textrm{ and}\\
        &\left\lVert \nabla \rho_\lambda(\theta_1)-\nabla \rho_\lambda(\theta_1)\right\rVert
        \leq  L_{\rho'} \left\lVert \theta_1 - \theta_2 \right\rVert;\\
        &\qquad L_{\rho'} = M_rM_e\left(M_h+M_eM_d^2\right)\\
        &\qquad\qquad +\lambda M_r^2M_e\left(3M_h+5 M_eM_d^2\right).
    \end{align*}
\end{lemma}
\begin{corollary}(MVRM-OnP-SF)
    \label{cr:mvrm_onP}
    Assume \ref{as:nabla_logpi}. Let the batch size $m=N>2$. Then, under the conditions of Theorem \ref{tm:sf}, we have
    \begin {align*}
    &\mathbb{E}\left[\left\lVert \nabla \rho_\lambda(\theta_R)\right\rVert^2\right]\\
    &\quad\leq \frac{2 \left(\rho_\lambda^* - \rho_\lambda(\theta_{0})\right)}{\sqrt{N}}
        + \frac{3d^2L_{\rho'} C_1}{2N} \\
        &\quad+ \frac{3d^2 \left(L_{\rho}^2L_{\rho'}+C_1+16e^2L_{\rho}^2\right )}{\sqrt{N}}+\frac{3 d^2 L_{\rho'}^2}{4\sqrt{N}}.
    \end {align*}
    In the above, $\rho_{\lambda}^*\!=\!\max_{\theta\in\mathbb{R}^d}\rho_{\lambda}(\theta)$. The constants $C_1, L_{\rho}$, and $L_{\rho'}$ are as in Lemmas \ref{lm:est_error_mvrm_pi} and \ref{lm:lip_rho_lambda}, respectively.
\end{corollary}
\begin{proof}
    Lemma \ref{lm:est_error_mvrm_pi} implies \ref{as:mse} holds for MVRM estimator.
    Lemma \ref{lm:lip_rho_lambda} implies the conditions in \ref{as:lip} and \ref{as:smooth} hold for MVRM.
    The main claim now follows by an application of Theorem \ref{tm:sf}.
\end{proof}
The bound for the off-policy variant follows by using
Lemma \ref{lm:est_error_mvrm_b} in place of Lemma \ref{lm:est_error_mvrm_pi}.
\begin{corollary}(MVRM-OffP-SF)
    \label{cr:mvrm_offP}
    Assume \ref{as:nabla_logpi}-\ref{as:b_pol}. Let the batch size $m=N>2$. Then, under the conditions of Theorem \ref{tm:sf}, we have
    \begin {align*}
    &\mathbb{E}\left[\left\lVert \nabla \rho_\lambda(\theta_R)\right\rVert^2\right]\\
    &\quad\leq \frac{2 \left(\rho_\lambda^* - \rho_\lambda(\theta_{0})\right)}{\sqrt{N}}
        + \frac{3d^2L_{\rho'} C_1}{2N} \\
        &\quad+ \frac{3d^2 \left(L_{\rho}^2L_{\rho'}+C_1+16e^2L_{\rho}^2\right )}{\sqrt{N}}+\frac{3 d^2 L_{\rho'}^2}{4\sqrt{N}}.
    \end {align*}
    In the above, $\rho_{\lambda}^*,L_{\rho}$, and $L_{\rho'}$ are as in Corollary \ref{cr:mvrm_onP}. The constant $C_1$ is as in Lemma \ref{lm:est_error_mvrm_b}.
\end{corollary}
\section{Conclusions}
\label{sec:conclusions}
We proposed two policy gradient algorithms that cater to the broad class of smooth risk measures. Both algorithms employed an SF-based gradient estimation scheme, and were shown to work in on-policy as well as off-policy RL settings. We derived non-asymptotic bounds that quantify the rate of convergence to our proposed algorithms to a stationary point of the smooth risk measure. As special cases, we showed that our theory and algorithms apply to optimization of MVRM and DRM, respectively. To the best of our knowledge, policy gradient algorithms with non-asymptotic convergence guarantees are not available in the literature for smooth risk measures in general, and the special cases of DRM and MVRM, in particular.

We focused exclusively on episodic MDPs that have a finite number of steps within each episode. Investigating the extension of our analysis beyond this limitation could be an interesting avenue for future research. As another potential future work, it would be interesting to investigate the convergence properties of non-smooth risk measures such as CVaR and CPT. While CVaR can be expressed as a DRM, its distortion function is not smooth, and CPT has a similar distortion function that is also non-smooth. To develop policy gradient algorithms, one could explore the possibility of using smooth approximations of these distortion functions and analyze their convergence properties.
\clearpage\newpage
\bibliography{vijayan}
\onecolumn 
\appendix

\section{Results for the policy gradient template}
\label{sec:SF_proofs}
\subsection{A vector version of Azuma-Hoeffding inequality}
\begin{lemma}
    \label{lm:hayes}
    Let $\{\hat{X}_i\}_{i=1}^n$ be i.i.d, vector-valued r.v.s., such that $\chi=\E\left[\hat{X}_i\right]$, $\forall \hat{X}_i$. Let $S_n=\frac{1}{n}\sum_{i=1}^n \hat{X}_i$. Assume $\forall i$, $\lVert \hat{X}_i \rVert \leq M$ a.s. Then,
    \begin{align*}
        \forall \epsilon>0,\; \p\left(\left\lVert S_n - \chi \right\rVert \geq \epsilon \right)\leq 2e^2 \exp \left( \frac{-n\epsilon^2}{8M^2}\right).
    \end{align*}
\end{lemma}
\begin{proof}
    Let
    \begin{align*}
        Y_{n'}= \begin{cases}\frac{1}{2M}\sum\limits_{i=1}^{n'} \left(\hat{X}_i - \chi\right),& \textrm{ for }n'=\{1,\cdots,n\}\\
            0,&\textrm{ for }n'=0.
        \end{cases}
    \end{align*}
    Then $\{Y_{n'}\}_{n'=1}^{n}$ is a set of partial sums of bounded mean zero r.v.s. Hence it is a martingale, and $\forall n'>0$
    \begin{align*}
        \left\lVert Y_{n'}-Y_{n'-1} \right\rVert &= \frac{1}{2M}\left\lVert\sum_{i=1}^{n'}\left(\hat{X}_i - \chi \right) - \sum_{i=1}^{n'-1}\left(\hat{X}_i - \chi \right) \right\rVert\\
        &= \frac{1}{2M}\left\lVert\hat{X}_{n'} - \chi \right\rVert\\
        &\leq \frac{1}{2M} 2M=1.
    \end{align*}
    Now,
    \begin{align*}
        \p\left(\left\lVert S_n - \chi \right\rVert \geq \epsilon\right)&=\p\left(\left\lVert Y_n\right\rVert\geq \frac{n\epsilon}{2M} \right)\\
        &\stackrel{(a)}{\leq} 2e^2 \exp \left( \frac{-n\epsilon^2}{8M^2}\right).
    \end{align*}
    In the above, the step $(a)$ follows from \cite[Theorem 1.8]{hayes2005large}, and the fact that every martingale is a very weak martingale (cf. \cite[Definition 1.3]{hayes2005large}).
\end{proof}
\subsection{Results with true objective function $\rho(\cdot)$}
The following lemmas establish some results related to the SF-based gradient estimate.
\begin{lemma}
    \label{lm:del_rho_mu_2}
    \begin{align*}
        \mathbb{E}\left[\widehat{\nabla}_{\mu,n}\rho(\theta)\right] = \nabla\rho_{\mu}(\theta).
    \end{align*}
\end{lemma}
\begin{proof}
    We follow the technique from \cite{shamir}. Since $v_{1:n}$ are i.i.d r.v.s, and have symmetric distribution around the origin, we obtain
    \begin {align*}
    \mathbb{E}\left[\widehat{\nabla}_{\mu,n}\rho(\theta)\right]
    &= \mathbb{E}\left[\frac{d}{n}\sum\limits_{i=1}^{n} \frac{\rho({\theta+\mu v_i}) - \rho({\theta - \mu v_i})}{2\mu}v_i\right]\\
    &=\frac{d}{2\mu n} \sum_{i=1}^{n}\mathbb{E}_{v}\left[\left(\rho({\theta+\mu v})- \rho({\theta-\mu v})\right)v\right] \\
    &=\frac{d}{2\mu} \left(\mathbb{E}_{v}\left[\rho({\theta+\mu v})v\right] + \mathbb{E}_{v}\left[\rho({\theta+\mu (- v)})(- v)\right]\right)\\
    &=\frac{d}{\mu}\mathbb{E}_{v}\left[\rho({\theta+\mu v})v\right] =\nabla \rho_{\mu}(\theta),
    \end {align*}
    where last equality follows from \cite[Lemma 2.1]{flaxman}.
\end{proof}
\begin{lemma}
    \label{lm:bias_rho}
    Assume \ref{as:smooth}. Then,
    \begin{align*}
        \left\lVert \nabla \rho_{\mu}(\theta) - \nabla \rho (\theta)\right\rVert \leq
        \frac{\mu d L_{\rho'}}{2}.
    \end{align*}
\end{lemma}
\begin{proof}
    The result follows from \cite[Proposition 7.5]{gao18}.
\end{proof}
    \begin{lemma}
        \label{lm:bias_SF_2}
        Assume \ref{as:lip}. Then,
        \begin{align*}
            \mathbb{E}\left[\left\lVert\widehat{\nabla}_{\mu,n} \rho({\theta}) - \nabla \rho_{\mu}(\theta)\right\rVert^2\right]
            \leq  \frac{16e^2d^2L_{\rho}^2}{n}.
        \end{align*}
    \end{lemma}
    \begin{proof}
        From \eqref{eq:hat_nabla_rho_0}, we obtain
        \begin{align}
            \label{eq:hat_nabla_rho_bound}
            \left\lVert\widehat{\nabla}_{\mu,n} \rho({\theta})\right\rVert
            &=\left\lVert\frac{d}{n}\sum\limits_{i=1}^{n} \frac{\rho({\theta+\mu v_i}) - \rho({\theta - \mu v_i})}{2\mu}v_i\right\rVert\nonumber\\
            &\stackrel{(a)}{\leq}\frac{d}{2\mu n} \sum\limits_{i=1}^{n}\left\lVert \rho({\theta+\mu v_i}) - \rho({\theta - \mu v_i})v_i\right\rVert\nonumber\\
            &\stackrel{(b)}{=}\frac{d}{2\mu n} \sum\limits_{i=1}^{n}\left\lvert \rho({\theta+\mu v_i}) - \rho({\theta - \mu v_i})\right\rvert\left\lVert v_i\right\rVert\nonumber\\
            &\stackrel{(c)}{=}\frac{d}{2\mu n} \sum\limits_{i=1}^{n}\left\lvert \rho({\theta+\mu v_i}) - \rho({\theta - \mu v_i})\right\rvert\nonumber\\
            &\stackrel{(d)}{\leq}\frac{d}{2\mu n} \sum\limits_{i=1}^{n}L_{\rho}\left\lVert 2 \mu v_i\right\rVert  \nonumber\\
            &\stackrel{(e)}{\leq}\frac{d}{n} \sum\limits_{i=1}^{n} L_{\rho} = d L_{\rho} \textrm{ a.s.},;\forall 0<n<\infty.
        \end{align}
        In the above, the step $(a)$ follows from the fact that $\lVert \sum_{i=1}^n a_i \rVert \leq  \sum_{i=1}^n \lVert a_i \rVert$, and the step $(b)$ follows from the fact that
        for a scalar $a$ and a vector $B$, $\lVert aB\rVert=\lvert a\rvert \lVert B\rVert $. The steps $(c)$ and $(e)$ follow since $\lVert v_i \rVert =1$. The step $(d)$ follows since
        $\forall \theta_1, \theta_2 \in \R^d$, $\left\lvert \rho(\theta_1)-\rho(\theta_2)\right\rvert \leq L_{\rho} \left\lVert \theta_1 - \theta_2 \right\rVert$ from \ref{as:lip}.

        From \eqref{eq:hat_nabla_rho_bound}, Lemma \ref{lm:del_rho_mu_2} and Lemma \ref{lm:hayes}, we obtain
        \begin{align}
            \label{eq:hat_nabla_rho_prob}
            \p\left(\left\lVert \widehat{\nabla}_{\mu,n} \rho({\theta})  - \nabla \rho_{\mu}(\theta)  \right\rVert > \epsilon\right) \leq 2e^2\exp\left(\frac{-n\epsilon^2}{8d^2 L_{\rho}^2}\right),
        \end{align}
        and
        \begin{align}
            \E\left[\left\lVert \widehat{\nabla}_{\mu,n} \rho({\theta})  - \nabla \rho_{\mu}(\theta) \right\rVert^2 \right]
            &= \int_{0}^{\infty} \p\left( \left\lVert \widehat{\nabla}_{\mu,n} \rho({\theta})  - \nabla \rho_{\mu}(\theta)\right\rVert > \sqrt{\epsilon}\right)d\epsilon \nonumber\\
            &\leq \int_{0}^{\infty} 2e^2\exp\left(-\frac{n\epsilon}{8d^2L_{\rho}^2}\right) d\epsilon \nonumber\\
            &= \frac{16e^2d^2L_{\rho}^2}{n}.\label{eq_s:hat_nabla_rho_err}
        \end{align}
    \end{proof}
\subsection{Results with approximate objective function $\hat{\rho}_m(\cdot)$}
The following lemmas establish bounds for the bias and variance of the gradient estimate in \eqref{eq:hat_nabla_hat_rho}.
\begin{lemma}
    \label{lm:bias_SF_1}
    Assume \ref{as:mse}. Then,
    \begin{align*}
        \E\left[\left\lVert\widehat{\nabla}_{\mu,n} \hat{\rho}_m({\theta}) -\widehat{\nabla}_{\mu,n}\rho({\theta}) \right\rVert^2\right]
        \leq  \frac{C_1d^2}{\mu^2m}.
    \end{align*}
\end{lemma}
\begin{proof}
    From \eqref{eq:hat_nabla_rho_0}  and \eqref{eq:hat_nabla_hat_rho}, we obtain
    \begin{align}
        &\mathbb{E}\left[\left\lVert\widehat{\nabla}_{\mu,n} \hat{\rho}_m({\theta}) -\widehat{\nabla}_{\mu,n}\rho({\theta}) \right\rVert^2\right]\nonumber\\
        &= \frac{d^2}{4n^2\mu^2}\mathbb{E}\left[\left\lVert \sum\limits_{i=1}^{n} \left( \left(\hat{\rho}_m({\theta+\mu v_i}) - \rho({\theta+\mu v_i})\right) +\left( \rho({\theta - \mu v_i})- \hat{\rho}_m({\theta - \mu v_i})\right)\right)v_i\right\rVert^2\right]\nonumber\\
        &\stackrel{(a)}{\leq} \frac{d^2}{4n^2\mu^2}\mathbb{E}\left[n\sum\limits_{i=1}^{n}\left\lVert  \left( \left(\hat{\rho}_m({\theta+\mu v_i}) - \rho({\theta+\mu v_i})\right) +\left( \rho({\theta - \mu v_i})- \hat{\rho}_m({\theta - \mu v_i})\right)\right)v_i\right\rVert^2\right]\nonumber\\
        &\stackrel{(b)}{=} \frac{d^2}{4\mu^2}\mathbb{E}\left[\left\lVert  \left( \left(\hat{\rho}_m({\theta+\mu v}) - \rho({\theta+\mu v})\right) +\left( \rho({\theta - \mu v})- \hat{\rho}_m({\theta - \mu v})\right)\right)v\right\rVert^2\right]\nonumber\\
        &\stackrel{(c)}{=} \frac{d^2}{4\mu^2}\mathbb{E}\left[\left\lvert  \left(\hat{\rho}_m({\theta+\mu v}) - \rho({\theta+\mu v})\right) +\left( \rho({\theta - \mu v})- \hat{\rho}_m({\theta - \mu v})\right)\right\rvert^2 \left\lVert v \right\rVert^2 \right]\nonumber\\
        &\stackrel{(d)}{=} \frac{d^2}{4\mu^2}\mathbb{E}\left[\left\lvert  \left(\hat{\rho}_m({\theta+\mu v}) - \rho({\theta+\mu v})\right) +\left( \rho({\theta - \mu v})- \hat{\rho}_m({\theta - \mu v})\right)\right\rvert^2 \right]\nonumber\\
        &\stackrel{(e)}{\leq} \frac{d^2}{2\mu^2}\mathbb{E}\left[\left\lvert \hat{\rho}_m({\theta+\mu v}) - \rho({\theta+\mu v})\right\rvert^2 +\left\lvert \rho({\theta - \mu v})- \hat{\rho}_m({\theta - \mu v})\right\rvert^2 \right]\nonumber\\
        &= \frac{d^2}{2\mu^2}\left(\mathbb{E}\left[\left\lvert \hat{\rho}_m({\theta+\mu v}) - \rho({\theta+\mu v})\right\rvert^2\right] +\mathbb{E}\left[\left\lvert \rho({\theta - \mu v})- \hat{\rho}_m({\theta - \mu v})\right\rvert^2 \right]\right)\nonumber\\
        &\stackrel{(f)}{\leq} \frac{C_1d^2}{\mu^2m}.
    \end{align}
    The step $(a)$ follows since $\lVert\sum_{i=1}^n a_i \rVert^2 \leq n\sum_{i=1}^n\lVert a_i \rVert^2$, and
    the step $(b)$ follows since $v_{1:n}$ are i.i.d r.v.s. The step $(c)$ follows from the fact that for a scalar $a$ and a vector $B$, $\lVert aB\rVert^2=\lvert a\rvert^2 \lVert B\rVert^2 $, and the step $(d)$ follows from the fact that $\lVert v \rVert = 1$. The step $(e)$ follows from the fact that $(a+b)^2 \leq 2 a^2 +2b^2$, and the step $(f)$ follows from the fact that $\E\left[\left\lvert \rho(\theta)- \hat{\rho}_m(\theta)\right\rvert^2\right]\leq\frac{C_1}{m}$ from \ref{as:mse}.
\end{proof}

\begin{lemma}
    \label{lm:bias_SF}
    Assume \ref{as:mse}-\ref{as:smooth}. Then,
    \begin{align*}
        \mathbb{E}\left[\left\lVert\widehat{\nabla}_{\mu,n} \hat{\rho}_m({\theta}) -\nabla\rho(\theta)\right\rVert^2\right]
        \leq \frac{3C_1d^2}{\mu^2m}+\frac{48e^2d^2L_{\rho}^2}{n}+\frac{3\mu^2 d^2 L_{\rho'}^2}{4}.
    \end{align*}
\end{lemma}
\begin{proof}
    Notice that
    \begin {align}
    &\mathbb{E}\left[\left\lVert\widehat{\nabla}_{\mu,n} \hat{\rho}_m({\theta}) -\nabla\rho(\theta)\right\rVert^2\right]\nonumber\\
    &=\mathbb{E}\left[\left\lVert\widehat{\nabla}_{\mu,n} \hat{\rho}_m({\theta})-\widehat{\nabla}_{\mu,n}\rho({\theta})+ \widehat{\nabla}_{\mu,n}\rho({\theta})-\nabla \rho_{\mu}(\theta)+\nabla \rho_{\mu}(\theta)-\nabla\rho(\theta)\right\rVert^2\right]\nonumber\\
    &\stackrel{(a)}{\leq}3\mathbb{E}\left[\left\lVert\widehat{\nabla}_{\mu,n} \hat{\rho}_m({\theta})-\widehat{\nabla}_{\mu,n}\rho({\theta})\right\rVert^2\right] +3\mathbb{E}\left[\left\lVert \widehat{\nabla}_{\mu,n}\rho({\theta})-\nabla \rho_{\mu}(\theta) \right\rVert^2 \right]\nonumber\\
    &\qquad+3\mathbb{E}\left[\left\lVert\nabla \rho_{\mu}(\theta)-\nabla\rho(\theta)\right\rVert^2\right]\nonumber\\
    &\stackrel{(b)}{\leq}\frac{3C_1d^2}{\mu^2m}+\frac{48e^2d^2L_{\rho}^2}{n}+\frac{3\mu^2 d^2 L_{\rho'}^2}{4}.
    \label{eq:bias_1}
    \end {align}

    In the above, the step $(a)$ follows from the fact that $\lVert a+b+c \rVert^2 \leq 3\lVert a\rVert^2 +3\lVert b \rVert^2 +3\lVert c \rVert^2$. The step $(b)$ follows from lemmas \ref{lm:bias_SF_1}, \ref{lm:bias_SF_2} and \ref{lm:bias_rho}.
\end{proof}
\begin{lemma}
    \label{lm:var_SF}
    Assume \ref{as:mse}-\ref{as:lip}. Then,
    \begin{align*}
        \mathbb{E}\left[\left\lVert \widehat{\nabla}_{\mu,n} \hat{\rho}_m({\theta}) \right\rVert^2\right]
        \leq \frac{3d^2C_1}{2\mu^2m} + \frac{3d^2L_{\rho}^2 }{4}.
    \end{align*}
\end{lemma}
\begin{proof}
    From \eqref{eq:hat_nabla_hat_rho}, we obtain
    \begin {align*}
    &\mathbb{E}\left[\left\lVert \widehat{\nabla}_{\mu,n} \hat{\rho}_m({\theta}) \right\rVert^2\right]\nonumber\\
    &= \mathbb{E}\left[\left\lVert\frac{d}{n}\sum\limits_{i=1}^{n} \frac{\hat{\rho}_m({\theta+\mu v_i}) - \hat{\rho}_m({\theta - \mu v_i})}{2\mu}v_i\right\rVert^2\right]\nonumber\\
    &\stackrel{(a)}{\leq} \frac{d^2}{4\mu^2n^2}\mathbb{E}\left[n\sum\limits_{i=1}^{n}\left\lVert \left (\hat{\rho}_m({\theta+\mu v_i}) - \hat{\rho}_m({\theta - \mu v_i})\right)v_i\right\rVert^2\right]\nonumber\\
    &\stackrel{(b)}{=} \frac{d^2}{4\mu^2n}\sum\limits_{i=1}^{n}\mathbb{E}\left[\left\lvert \hat{\rho}_m({\theta+\mu v_i}) - \hat{\rho}_m({\theta - \mu v_i})\right\rvert^2 \left\lVert v_i\right\rVert^2\right]\nonumber\\
    &\stackrel{(c)}{=} \frac{d^2}{4\mu^2n}\sum\limits_{i=1}^{n}\mathbb{E}\left[\left\lvert \hat{\rho}_m({\theta+\mu v_i}) - \hat{\rho}_m({\theta - \mu v_i})\right\rvert^2 \right]\nonumber\\
    &\stackrel{(d)}{=} \frac{d^2}{4\mu^2}\mathbb{E}\left[\left\lvert \hat{\rho}_m({\theta+\mu v}) - \hat{\rho}_m({\theta - \mu v})\right\rvert^2 \right]\nonumber\\
    &= \frac{d^2}{4\mu^2}\mathbb{E}\left[\left\lvert \hat{\rho}_m({\theta+\mu v})-\rho({\theta+\mu v}) +\rho({\theta-\mu v}) - \hat{\rho}_m({\theta - \mu v}) \right.\right.\nonumber\\
    &\qquad \left.\left.+ \rho({\theta+\mu v})- \rho({\theta-\mu v}) \right\rvert^2 \right]\nonumber\\
    &\stackrel{(e)}{\leq} \frac{3d^2}{4\mu^2} \left (\E\left[\left\lvert \hat{\rho}_m({\theta+\mu v})-\rho({\theta+\mu v})\right\rvert^2 \right]+ \E\left[\left\lvert \rho({\theta-\mu v}) - \hat{\rho}_m({\theta - \mu v})\right\rvert^2 \right]\right.\nonumber\\
    &\qquad \left.+ \E\left[\left\lvert \rho({\theta+\mu v})- \rho({\theta-\mu v}) \right\rvert^2 \right]\right)\nonumber\\
    &\stackrel{(f)}{\leq}  \frac{3d^2}{4\mu^2} \left(\frac{2C_1}{m}+ 4\mu^2 L_{\rho}^2 \right)=\frac{3d^2C_1}{2\mu^2m} + 3d^2L_{\rho}^2.
    \end {align*}

    In the above, the step $(a)$ follows from the fact that $\lVert\sum_{i=1}^n a_i \rVert^2 \leq n\sum_{i=1}^n\lVert a_i \rVert^2$, and the step $(b)$ follows from the fact that for a scalar $a$ and a vector $B$, $\lVert aB\rVert^2=\lvert a\rvert^2 \lVert B\rVert^2 $. The step $(c)$ follows since $\left\lVert v_i\right\rVert = 1$, and step $(d)$ follows since $v_i's$ are i.i.d r.v.s. The step $(e)$ follows from the fact that $\lvert\sum_{i=1}^n a_i \rvert^2 \leq n\sum_{i=1}^n\lvert a_i \rvert^2$. The step $(f)$ follows since $\forall \theta_1, \theta_2 \in \R^d$, $\E\left[\left\lvert \rho(\theta)- \hat{\rho}_m(\theta)\right\rvert^2\right]\leq\frac{C_1}{m}$, $\left\lvert \rho(\theta_1)-\rho(\theta_2)\right\rvert \leq L_{\rho} \left\lVert \theta_1 - \theta_2 \right\rVert$ from \ref{as:mse}-\ref{as:lip}, and $\left\lVert v\right\rVert^2 =1$.
\end{proof}
\subsection{Proof of Proposition \ref{pr:non_asym_sf}}
Using the fundamental theorem of calculus, we obtain
\begin {align}
& \rho(\theta_k) - \rho(\theta_{k+1})
=\langle \nabla \rho(\theta_k), \theta_k - \theta_{k+1} \rangle
+ \int_0^1 \left\langle \nabla \rho(\theta_{k+1}+\tau(\theta_k-\theta_{k+1}))-\nabla \rho(\theta_k), \theta_k - \theta_{k+1} \right\rangle d\tau\nonumber\\
&\leq\langle \nabla \rho(\theta_k), \theta_k - \theta_{k+1} \rangle
+\int_0^1 \left\lVert\nabla\rho(\theta_{k+1}+\tau(\theta_k - \theta_{k+1})) - \nabla\rho(\theta_k) \right\rVert \left\lVert \theta_k - \theta_{k+1} \right\rVert d\tau\nonumber\\
&\stackrel{(a)}{\leq} \left \langle \nabla \rho(\theta_k), \theta_k - \theta_{k+1} \right \rangle
+ L_{\rho'}\left\lVert \theta_k - \theta_{k+1} \right\rVert^2  \int_0^1 (1-\tau) d\tau \nonumber\\
&= \left \langle \nabla \rho(\theta_k), \theta_k - \theta_{k+1} \right \rangle + \frac{L_{\rho'}}{2}\left\lVert \theta_k - \theta_{k+1} \right\rVert^2 \nonumber\\
&= \alpha\left\langle \nabla \rho(\theta_k),-\widehat{\nabla}_{\mu,n}\hat{\rho}_m(\theta_k) \right\rangle
+ \frac{L_{\rho'}}{2}\alpha^2 \left\lVert\widehat{\nabla}_{\mu,n}\hat{\rho}_m(\theta_k)\right\rVert^2\nonumber\\
&= \alpha\left \langle \nabla \rho(\theta_k), \nabla \rho(\theta_k)-\widehat{\nabla}_{\mu,n}\hat{\rho}_m(\theta_k) \right \rangle
-\alpha\left \lVert \nabla \rho(\theta_k)\right \rVert^2
+ \frac{L_{\rho'}}{2}\alpha^2\left\lVert \widehat{\nabla}_{\mu,n}\hat{\rho}_m(\theta_k) \right\rVert^2 \nonumber\\
&\stackrel{(b)}{\leq} \frac{\alpha}{2}\left \lVert \nabla \rho(\theta_k) \right \rVert^2 + \frac{\alpha}{2}\left \lVert \nabla \rho(\theta_k) -\widehat{\nabla}_{\mu,n}\hat{\rho}_m(\theta_k) \right \rVert^2
-\alpha\left \lVert \nabla \rho(\theta_k)\right \rVert^2+ \frac{L_{\rho'}}{2}\alpha^2\left\lVert \widehat{\nabla}_{\mu,n}\hat{\rho}_m(\theta_k) \right\rVert^2\nonumber\\
&= \frac{\alpha}{2}\left \lVert \nabla \rho(\theta_k) -\widehat{\nabla}_{\mu,n}\hat{\rho}_m(\theta_k) \right \rVert^2
-\frac{\alpha}{2}\left \lVert \nabla \rho(\theta_k)\right \rVert^2
+ \frac{L_{\rho'}}{2}\alpha^2\left\lVert \widehat{\nabla}_{\mu,n}\hat{\rho}_m(\theta_k) \right\rVert^2.
\label{eq:sf_1}
\end {align}
In the above the step $(a)$ follows since $\rho(\cdot)$ is smooth and the step $(b)$ follows from $2\langle a, b \rangle \leq \lVert a \rVert^2+ \Vert b \rVert^2$.
Rearranging  and taking expectations on both sides of \eqref{eq:sf_1}, we obtain
\begin {align}
&\alpha\mathbb{E}\left[\left \lVert \nabla \rho(\theta_k)\right \rVert^2\right]\nonumber\\
&\leq  2\mathbb{E}\left[\rho(\theta_{k+1}) - \rho(\theta_{k})\right] + L_{\rho'}\alpha^2 \mathbb{E}\left[\left\lVert \widehat{\nabla}_{\mu,n}\hat{\rho}_m(\theta_k) \right\rVert^2\right]
+  \alpha \mathbb{E}\left[\left \lVert \nabla \rho(\theta_k) -\widehat{\nabla}_{\mu,n}\hat{\rho}_m(\theta_k) \right \rVert^2 \right]\nonumber\\
&\leq  2\mathbb{E}\left[\rho(\theta_{k+1}) - \rho(\theta_{k})\right] +
L_{\rho'}\alpha^2 \left (\frac{3d^2C_1}{2\mu^2m} + 3d^2L_{\rho}^2 \right ) +  \alpha\left( \frac{3C_1d^2}{\mu^2m}+\frac{48e^2d^2L_{\rho}^2}{n}+\frac{3\mu^2 d^2 L_{\rho'}^2}{4} \right)
\label{eq:sf_2}
\end {align}
where the last inequality follows from lemmas \ref{lm:bias_SF}-\ref{lm:var_SF}.

Summing up \eqref{eq:sf_2} from $k=0,\cdots,N-1$, we obtain
\begin {align*}
&\alpha\sum\limits_{k=0}^{N-1}\mathbb{E}\left[\left \lVert \nabla \rho(\theta_k)\right \rVert^2\right]\\
&\quad\leq  2 \mathbb{E}\left[\rho(\theta_{N}) - \rho(\theta_{0})\right]
+N L_{\rho'} \alpha^2
\left (\frac{3d^2C_1}{2\mu^2m} + 3d^2L_{\rho}^2\right )
+ N \alpha \left( \frac{3C_1d^2}{\mu^2m}+\frac{48e^2d^2L_{\rho}^2}{n}+\frac{3\mu^2 d^2 L_{\rho'}^2}{4} \right).
\end {align*}
Since $\theta_R$ is chosen uniformly at random from the policy iterates $\{\theta_0,\cdots,\theta_{N-1}\}$, we obtain
\begin {align*}
\mathbb{E}\left[\left\lVert \nabla \rho(\theta_R)\right\rVert^2\right]
&= \frac{1}{N}\sum\limits_{k=0}^{N-1}\mathbb{E}\left[\left\lVert \nabla \rho(\theta_k)\right \rVert^2\right]\\
&\leq \frac{2 \left(\rho^* - \rho(\theta_{0})\right)}{N \alpha}
L_{\rho'} \alpha \left (\frac{3d^2C_1}{2\mu^2m} + 3d^2L_{\rho}^2\right )
+ \left( \frac{3C_1d^2}{\mu^2m}+\frac{48e^2d^2L_{\rho}^2}{n}+\frac{3\mu^2 d^2 L_{\rho'}^2}{4} \right).
\end {align*}
\hfill\qed

\section{DRM}
\label{sec:drm_proofs}
\subsection{Estimating DRM using Order statistics}
\label{subsec:est_drm}
The following lemma estimates the DRM in an on-policy RL setting.
\begin{lemma}
    \label{lm:hat_rho_G}
    $\hat{\rho}_g^G(\theta)=  \sum\limits_{i=1}^{m} {R^\theta_{(i)}} \left(g\left(1- \frac{i-1}{m}\right) - g\left(1- \frac{i}{m}\right)\right)$.
\end{lemma}
\begin{proof}
    Our proof follows the technique from \cite{kim_2010}. We rewrite \eqref{eq:G} as
    \begin {align}
    \label{eq:G2}
    G^m_{R^{\theta}}(x) =
    \begin{cases}
        0,&\textrm{if } x < R^\theta_{(1)}\\
        \frac{i}{m},&\textrm{if } R^\theta_{(i)} \leq x < R^\theta_{(i+1)}, i\in\{1,\!\cdots\!,m-1\}\\
        1,&\textrm{if } x \geq R^\theta_{(m)},\\
    \end{cases}
    \end {align}
    where $R^\theta_{(i)}$ is the $i^{th}$ smallest order statistic from the samples $R^\theta_1,\cdots R^\theta_m$.

    We assume without loss of generality that $R^\theta_{(j)} < 0 < R^\theta_{(j+1)}$, and obtain,
    \begin {align*}
    &\hat{\rho}_g^G(\theta)=\int\limits_{-M_r}^{0}(g(1-G^m_{R^{\theta}}(x))-1) dx + \int\limits_{0}^{M_r}g(1-G^m_{R^{\theta}}(x))dx\\
    &=\int\limits_{-M_r}^{R^\theta_{(1)}} (g(1-G^m_{R^{\theta}}(x))-1)dx
    + \sum_{i=2}^j \int\limits_{R^\theta_{(i-1)}}^{R^\theta_{(i)}} (g(1-G^m_{R^{\theta}}(x))-1)dx
    + \int\limits_{R^\theta_{(j)}}^{0}(g(1-G^m_{R^{\theta}}(x))-1)dx\\
    &\quad+ \int\limits_{0}^{R^\theta_{(j+1)}} g(1-G^m_{R^{\theta}}(x))dx
    + \sum_{i=j+1}^{m-1}\int\limits_{R^\theta_{(i)}}^{R^\theta_{(i+1)}} g(1-G^m_{R^{\theta}}(x))dx+ \int\limits_{R^\theta_{(m)}}^{M_r}g(1-G^m_{R^{\theta}}(x))dx\\
    &= \sum_{i=2}^j \int\limits_{R^\theta_{(i-1)}}^{R^\theta_{(i)}} \left(g\left(1- \frac{i-1}{m}\right)-1\right)dx
    + \int\limits_{R^\theta_{(j)}}^{0}\left(g\left(1- \frac{j}{m}\right)-1\right)dx+ \int\limits_{0}^{R^\theta_{(j+1)}} g\left(1- \frac{j}{m}\right)dx\\
    &\quad + \sum_{i=j+1}^{m-1}\int\limits_{R^\theta_{(i)}}^{R^\theta_{(i+1)}} g\left( 1-\frac{i}{m}\right)dx\\
    &= \sum_{i=2}^j \left({R^\theta_{(i)}}-{R^\theta_{(i-1)}}\right) \left(g\left(1- \frac{i-1}{m}\right)-1\right)
    - {R^\theta_{(j)}}\left(g\left(1- \frac{j}{m}\right)-1\right)+ {R^\theta_{(j+1)}} g\left(1- \frac{j}{m}\right)\\
    &\quad+ \sum_{i=j+1}^{m-1} \left({R^\theta_{(i+1)}}-{R^\theta_{(i)}}\right) g\left( 1\!-\!\frac{i}{m}\right)\\
    &= \sum_{i=2}^j \left({R^\theta_{(i)}}-{R^\theta_{(i-1)}}\right) g\left(1- \frac{i-1}{m}\right)+ {R^\theta_{(1)}}
    + \sum_{i=j}^{m-1} \left({R^\theta_{(i+1)}}-{R^\theta_{(i)}}\right) g\left( 1-\frac{i}{m}\right) \\
    &= \sum_{i=1}^{m} {R^\theta_{(i)}} g\left(1- \frac{i-1}{m}\right) - \sum_{i=1}^{m-1}{R^\theta_{(i)}} g\left(1- \frac{i}{m}\right)\\
    &= \sum_{i=1}^{m} {R^\theta_{(i)}} \left(g\left(1- \frac{i-1}{m}\right) - g\left(1- \frac{i}{m}\right)\right).
    \end {align*}
\end{proof}
The following lemma estimates the DRM in an off-policy RL setting.
\begin{lemma}
    \label{lm:hat_rho_H}
    $\hat{\rho}_g^H(\theta)= R^b_{(1)}+\sum\limits_{i=2}^{m} {R^b_{(i)}} g\left(1- \min\left\{1,\frac{1}{m}\sum\limits_{k=1}^{i-1}\psi^\theta_{(k)}\right\}\right)
    - \sum\limits_{i=1}^{m-1}{R^b_{(i)}} g\left(1- \min\left\{1,\frac{1}{m}\sum\limits_{k=1}^{i}\psi^\theta_{(k)}\right\}\right)$.
\end{lemma}
\begin{proof}
    We rewrite \eqref{eq:H} as
    \begin {align}
    \label{eq:H2}
    H^m_{R^{\theta}}(x) =
    \begin{cases}
        0,&\textrm{if } x < R^b_{(1)}\\
        min\{1,\frac{1}{m}\sum\limits_{j=1}^{i}\psi^\theta_{(j)}\},&\textrm{if } R^b_{(i)} \leq x < R^\theta_{(i+1)},i\in\{1,\!\cdots\!,m-1\}\\
        1,&\textrm{if } x \geq R^b_{(m)},\\
    \end{cases}
    \end {align}
    where $R^b_{(i)}$ is the $i^{th}$ smallest order statistic from the samples $R^b_1,\cdots R^b_m$, and $\psi^\theta_{(i)}$ is the importance sampling ratio of $R^b_{(i)}$.

    We assume without loss of generality that $R^b_{(j)} < 0 < R^b_{(j+1)}$, and obtain,
    \begin {align*}
    &\hat{\rho}_g^H(\theta)=\int\limits_{-M_r}^{0}(g(1-H^m_{R^{\theta}}(x))-1) dx + \int\limits_{0}^{M_r}g(1-H^m_{R^{\theta}}(x))dx\\
    &=\int\limits_{-M_r}^{R^b_{(1)}} (g(1-H^m_{R^{\theta}}(x))-1)dx
    + \sum_{i=2}^j \int\limits_{R^b_{(i-1)}}^{R^b_{(i)}} (g(1-H^m_{R^{\theta}}(x))-1)dx
    + \int\limits_{R^b_{(j)}}^{0}(g(1-H^m_{R^{\theta}}(x))-1)dx\\
    &\quad+ \int\limits_{0}^{R^b_{(j+1)}} g(1-H^m_{R^{\theta}}(x))dx
    + \sum_{i=j+1}^{m-1}\int\limits_{R^b_{(i)}}^{R^b_{(i+1)}} g(1-H^m_{R^{\theta}}(x))dx+ \int\limits_{R^b_{(m)}}^{M_r}g(1-H^m_{R^{\theta}}(x))dx\\
    &= \sum_{i=2}^j \int\limits_{R^b_{(i-1)}}^{R^b_{(i)}} \left(g\left(1- min\left\{1,\frac{1}{m}\sum_{k=1}^{i-1}\psi^\theta_{(k)}\right\}\right)-1\right)dx
    + \int\limits_{R^b_{(j)}}^{0}\left(g\left(1- min\left\{1,\frac{1}{m}\sum_{k=1}^{j}\psi^\theta_{(k)}\right\}\right)-1\right)dx\\
    &\quad+ \int\limits_{0}^{R^b_{(j+1)}} g\left(1- \min\left\{1,\frac{1}{m}\sum_{k=1}^{j}\psi^\theta_{(k)}\right\}\right)dx
    + \sum_{i=j+1}^{m-1}\int\limits_{R^b_{(i)}}^{R^b_{(i+1)}} g\left( 1-\min\left\{1,\frac{1}{m}\sum_{k=1}^{i}\psi^\theta_{(k)}\right\}\right)dx\\
    &= \sum_{i=2}^j \left({R^b_{(i)}}-{R^b_{(i-1)}}\right) \left(g\left(1- \min\left\{1,\frac{1}{m}\sum_{k=1}^{i-1}\psi^\theta_{(k)}\right\}\right)-1\right)
    - {R^b_{(j)}}\left(g\left(1- \min\left\{1,\frac{1}{m}\sum_{k=1}^{j}\psi^\theta_{(k)}\right\}\right)-1\right)\\
    &\quad + {R^b_{(j+1)}} g\left(1- \min\left\{1,\frac{1}{m}\sum_{k=1}^{j}\psi^\theta_{(k)}\right\}\right)
    + \sum_{i=j+1}^{m-1} \left({R^b_{(i+1)}}-{R^b_{(i)}}\right) g\left( 1-\min\left\{1,\frac{1}{m}\sum_{k=1}^{i}\psi^\theta_{(k)}\right\}\right)\\
    &= \sum_{i=2}^j \left({R^b_{(i)}}-{R^b_{(i-1)}}\right) g\left(1-\min\left\{1, \frac{1}{m}\sum_{k=1}^{i-1}\psi^\theta_{(k)}\right\}\right)+ {R^b_{(1)}}\\
    &\quad+ \sum_{i=j}^{m-1} \left({R^b_{(i+1)}}-{R^b_{(i)}}\right) g\left( 1-\min\left\{1,\frac{1}{m}\sum_{k=1}^{i}\psi^\theta_{(k)}\right\}\right) \\
    &= R^b_{(1)}+\sum_{i=2}^{m} {R^b_{(i)}} g\left(1- \min\left\{1,\frac{1}{m}\sum_{k=1}^{i-1}\psi^\theta_{(k)}\right\}\right)
    - \sum_{i=1}^{m-1}{R^b_{(i)}} g\left(1- \min\left\{1,\frac{1}{m}\sum_{k=1}^{i}\psi^\theta_{(k)}\right\}\right).
    \end {align*}
\end{proof}
\subsection{The estimation error of the DRM}
In the following lemma, we bound the estimation error of the DRM in an on-policy RL setting.
\begin{proof}(\textbf{Lemma \ref{lm:est_error_G}})
    Since $\forall x\in(-M_r,M_r),\left\lvert\1\{R^{\theta}\leq x\}\right\rvert  \leq  1$ a.s., using Hoeffding's inequality, we obtain $\forall x\in(-M_r,M_r)$,
    \begin {align}
    &\p\left(\left\lvert G^m_{R^{\theta}}(x) - F_{R^{\theta}}(x) \right\rvert > \epsilon\right) \leq 2\exp\left(-\frac{m\epsilon^2}{2}\right), \textrm{ and} \nonumber\\
    &\E\left[\left\lvert G^m_{R^{\theta}}(x) - F_{R^{\theta}}(x)\right\rvert^2 \right]
    =\int_{0}^{\infty}\p\left(\left\lvert G^m_{R^{\theta}}(x) -F_{R^{\theta}}(x)\right\rvert > \sqrt{\epsilon}\right)d\epsilon
    \leq \int_{0}^{\infty} 2\exp\left(-\frac{m\epsilon}{2}\right) d\epsilon = \frac{4 }{m}.\label{eq:G_err}
    \end {align}
    Now,
    \begin {align}
    &\E\left[\left\lvert \rho_g(\theta)- \hat{\rho}_g^G(\theta)\right\rvert^2\right]
    =\E\left[\left\lvert\int_{-M_r}^{M_r}(g(1-F_{R^{\theta}}(x))- g(1-G^m_{R^{\theta}}(x))) dx\right\rvert^2 \right]\nonumber\\
    &\stackrel{(a)}{\leq}2M_r\E\left[\int_{-M_r}^{M_r}\left\lvert(g(1-F_{R^{\theta}}(x))- g(1-G^m_{R^{\theta}}(x))) \right\rvert^2 dx\right]
    \nonumber\\
    &\stackrel{(b)}{\leq}2M_r\int_{-M_r}^{M_r}\E\left[\left\lvert(g(1-F_{R^{\theta}}(x))- g(1-G^m_{R^{\theta}}(x))) \right\rvert^2 \right]dx
   \nonumber\\
    &\stackrel{(c)}{\leq}2M_rM_{g'}^2\!\int_{-M_r}^{M_r}\E\left[\left\lvert G^m_{R^{\theta}}(x) - F_{R^{\theta}}(x) \right\rvert^2 \right]dx\nonumber\\
    &\stackrel{(d)}{\leq}2M_rM_{g'}^2\int_{-M_r}^{M_r}\frac{4 }{m}dx
    =\frac{16M_r^2M_{g'}^2}{m},\label{eq:est_error_G}
    \end {align}
    where \((a)\) follows from the Cauchy-Schwarz inequality, \((b)\) follows from the Fubini's theorem, \((c)\) follows from Lemma \ref{lm:g_lip}, and \((d)\) follows from \eqref{eq:G_err}.
\end{proof}
In the following lemma, we bound the estimation error of the DRM in an off-policy RL setting.
\begin{proof}(\textbf{Lemma \ref{lm:est_error_H}})
    We use parallel arguments to the proof of Lemma \ref{lm:est_error_G}.

    From \eqref{eq:is_ratio}, we obtain $\forall x \in(-M_r,M_r)$, $\left\lvert\1\{R^{\theta}\leq x\}\psi^\theta\right\rvert  \leq  M_s$ a.s.
    From Hoeffding inequality, we obtain $\forall x \in(-M_r,M_r)$,
    \begin {align}
    \label{eq:hatH_prob}
    \p\left(\left\lvert \hat{H}^m_{R^{\theta}}(x) - F_{R^{\theta}}(x) \right\rvert > \epsilon\right) \leq 2\exp\left(-\frac{m\epsilon^2}{2M_s^2}\right).
    \end {align}
    From \eqref{eq:H} and \eqref{eq:hatH}, we observe that $\p\left(\left\lvert H^m_{R^{\theta}}(x) - F_{R^{\theta}}(x) \right\rvert > \epsilon\right) \leq \p\left(\left\lvert \hat{H}^m_{R^{\theta}}(x) - F_{R^{\theta}}(x) \right\rvert > \epsilon\right)$. Hence, we obtain $\forall x \in (-M_r,M_r)$,
    \begin {align}
    \label{eq:H_prob}
    \p\left(\left\lvert H^m_{R^{\theta}}(x) - F_{R^{\theta}}(x) \right\rvert > \epsilon\right) \leq 2\exp\left(-\frac{m\epsilon^2}{2M_s^2}\right).
    \end {align}
    Using similar arguments as in \eqref{eq:G_err} along with \eqref{eq:H_prob}, we obtain $\forall x \in[-M_r,M_r]$,
    \begin {align}
    \label{eq:H_err}
    \E\left[\left\lvert H^m_{R^{\theta}}(x) - F_{R^{\theta}}(x)\right\rvert^2 \right]\leq\frac{4 M_s^2}{m},\forall x.
    \end {align}
    Using similar arguments as in \eqref{eq:est_error_G} along with \eqref{eq:H_err}, we obtain
    \begin {align*}
    \E\left[\left\lvert \rho_g(\theta)- \hat{\rho}_g^H(\theta)\right\rvert^2\right]=\frac{16M_r^2M_{g'}^2M_s^2}{m}.
    \end {align*}
\end{proof}
\subsection{Lipschitz properties of the DRM and its gradient}
\label{subsec:lip_drm}
\subsubsection{Results related to the distortion function}
The following lemma establishes Lipschitzness of the $g(\cdot)$, and $g'(\cdot)$. We require this result to establish the smoothness of the DRM.
\begin{lemma}
    \label{lm:g_lip}
    $\forall t,t'\in(0,1)$,$\left\lvert g(t) -g(t')\right\rvert \leq M_{g'} \left\lvert t - t'\right\rvert$, and
    $\left\lvert g'(t) -g'(t')\right\rvert \leq  M_{g''}\left\lvert t - t'\right\rvert$.
\end{lemma}
\begin{proof}
    Using mean value theorem, we obtain $ g(t) -g(t') = g'(\tilde{t}) (t - t')$, where $\tilde{t} \!\in\! (t,t')$. From \ref{as:g'_bound}, we obtain  $\left\lvert g'(\tilde{t}) \right\rvert \leq M_{g'}, \forall \tilde{t} \in(0,1)$. Hence,
    $\left\lvert g(t) -g(t')\right\rvert \leq M_{g'} \left\lvert t - t'\right\rvert\;\forall t,t'\in(0,1)$.

    Similarly, we have $ g'(t)-g'(t') = g''(\tilde{t}) (t - t')$, where $\tilde{t} \in (t,t')$. From \ref{as:g'_bound}, we obtain  $\left\lvert g''(\tilde{t}) \right\rvert \leq M_{g''}, \forall \tilde{t} \in(0,1)$. Hence,
    $\left\lvert g'(t) -g'(t')\right\rvert \leq M_{g''} \left\lvert t - t'\right\rvert\;\forall t,t'\in(0,1)$.
\end{proof}
\subsubsection{Lipschitz properties of the CDF}
The following two lemmas establish an upper bound for the gradient and the Hessian of the CDF. These lemmas are similar to lemmas in \cite{nv2021}. For the sake of completeness, we provide the detailed proof.
\begin{lemma}
    \label{lm:nablaFG}
    $\forall x \in(-M_r,M_r)$,
    \begin{align*}
        &\nabla F_{R^{\theta}}(x) =\E\left[\1\{R^{\theta} \leq x\}\sum\limits_{t=0}^{T-1}\nabla\log\pi_\theta(A_t | S_t)\right], \textrm{ and}\\
        &\nabla^2 F_{R^{\theta}}(x) =\E\left[\1\{R^{\theta}\leq x\}\left(\sum_{t=0}^{T-1}\nabla^2\log\pi_\theta(A_t | S_t)+
        \left[\sum_{t=0}^{T-1}\nabla\log\pi_\theta(A_t | S_t)\right] \left[\sum_{t=0}^{T-1}\nabla\log\pi_\theta(A_t | S_t)\right]^T\right)\right].
    \end{align*}
\end{lemma}
\begin{proof}
    Let $\Omega$ denote the set of all sample episodes. For any episode $\omega\in\Omega$, we denote by $T(\omega)$, its length, and $S_t(\omega)$ and $A_t(\omega)$, the state and action at time $t\in\{0,1,2,\cdots\}$ respectively. \\
    Let $R(\omega)=\sum\limits_{t=0}^{T(\omega)-1}\gamma^t r(S_t(\omega),A_t(\omega),S_{t+1}(\omega))$ be the cumulative discounted reward of the episode $\omega$, and let\\ $\p_\theta(\omega) =\prod\limits_{t=0}^{T(\omega)-1}\pi_\theta(A_t(\omega)|S_t(\omega))p(S_{t+1}(\omega),S_t(\omega),A_t(\omega))$. \\
    From $\frac{\nabla \p_\theta(\omega)}{\p_\theta(\omega)} =\sum\limits_{t=0}^{T(\omega)-1}\nabla\log\pi_\theta(A_t(\omega) | S_t(\omega))$,
    we obtain
    \begin{align*}
        \nabla F_{R^{\theta}}(x)
        &=\nabla\E\!\left[\1\{R^{\theta}\leq x\}\right] =\nabla \sum_{\omega\in\Omega} \1\{R(\omega)\leq x\}\p_\theta(\omega)\\
        &\stackrel{(a)}{=}\sum_{\omega\in\Omega}\nabla\left( \1\{R(\omega)\leq x\}\p_\theta(\omega)\right)\\
        &\stackrel{(b)}{=}\sum_{\omega\in\Omega} \1\{R(\omega)\leq x\}\nabla \p_\theta(\omega)\\
        &=\sum_{\omega\in\Omega} \1\{R(\omega)\leq x\} \frac{\nabla \p_\theta(\omega)}{\p_\theta(\omega)}\p_\theta(\omega)\\
        &=\sum_{\omega\in\Omega} \1\{R(\omega)\leq x\} \sum_{t=0}^{T(\omega)-1}\nabla\log\pi_\theta(A_t(\omega)|S_t(\omega))\p_\theta(\omega)\\
        &=\E\left[\1\{R^{\theta}\leq x\}\sum_{t=0}^{T-1}\nabla\log\pi_\theta(A_t|S_t)\right].
    \end{align*}
    In the above, the equality in \((a)\) follows by an application of the dominated convergence theorem to interchange the differentiation and the expectation operation. The aforementioned application is allowed since  (i) $\Omega$ is finite and the underlying measure is bounded, as we consider an MDP where the state and actions spaces are finite, and the policies are proper, (ii) $\nabla\log\pi_\theta(A_t|S_t)$ is bounded from \ref{as:nabla_logpi}. The equality in \((b)\) follows, since for a given episode $\omega$, the cumulative reward $R(\omega)$ does not depend on $\theta$.

    Similarly, \\
    from $\frac{\nabla^2 \p_\theta(\omega)}{\p_\theta(\omega)} =\sum\limits_{t=0}^{T(\omega)-1}\nabla^2\log\pi_\theta(A_t(\omega) | S_t(\omega))+
    \left[\sum\limits_{t=0}^{T(\omega)-1}\nabla\log\pi_\theta(A_t(\omega) | S_t(\omega))\right] \left[\sum\limits_{t=0}^{T(\omega)-1}\nabla\log\pi_\theta(A_t(\omega) | S_t(\omega))\right]^T$, we obtain
    \begin{align*}
        &\nabla^2 F_{R^{\theta}}(x)
        =\E\left[\1\{R^{\theta}\leq x\}\left(\sum_{t=0}^{T-1}\!\nabla^2\log\pi_\theta(A_t | S_t)+
        \left[\sum_{t=0}^{T-1}\nabla\log\pi_\theta(A_t | S_t)\right] \left[\sum_{t=0}^{T-1}\nabla\log\pi_\theta(A_t | S_t)\right]^T\right)\right].
    \end{align*}
\end{proof}
\begin{lemma}
    \label{lm:nablaF_bound}
    $\forall x \in (-M_r,M_r),\left\lVert \nabla F_{R^{\theta}}(x) \right\rVert \leq M_e M_d$, and
    $\left\lVert \nabla^2 F_{R^{\theta}}(x) \right\rVert\leq M_e M_h +M_e^2M_d^2$.
\end{lemma}
\begin{proof}
    From \ref{as:nabla_logpi} and \eqref{eq:M_e}, for any $x\in(-M_r,M_r)$, we have
    \begin{align}
        \label{eq:nabla_G_bound}
        \left\lVert\1\{R^{\theta}\leq x\}\sum_{t=0}^{T-1}\!\nabla\log \pi_{\theta}(A_t\mid S_t)\right\rVert \leq  M_e M_d \textrm{ a.s},
    \end{align}
    and
    \begin{align}
        \label{eq:nabla_G_bound1}
        \left\lVert\1\{R^{\theta}\leq x\}\left(\sum_{t=0}^{T-1}\nabla^2\log\pi_\theta(A_t | S_t)+
        \left[\sum_{t=0}^{T-1}\nabla\log\pi_\theta(A_t | S_t)\right] \left[\sum_{t=0}^{T-1}\nabla\log\pi_\theta(A_t | S_t)\right]^T\right)\right\rVert
        \leq  M_e M_h +M_e^2M_d^2\textrm{ a.s}.
    \end{align}

    From Lemma \ref{lm:nablaFG}, for any $x\in(-M_r,M_r)$, we have
    \begin{align}
        \left\lVert \nabla F_{R^{\theta}}(x) \right\rVert
        &\leq \E\left[\left\lVert \1\{R^{\theta}\leq x\} \sum_{t=0}^{T-1} \nabla\log \pi_{\theta}(A_t| S_t)\right\rVert\right]
        \leq M_e M_d, \label{eq:nabla_F_R_bound}
    \end{align}
    and
    \begin{align}
        \left\lVert \nabla^2 F_{R^{\theta}}(x) \right\rVert
        &\leq \E\left[\left\lVert\1\{R^{\theta}\leq x\}\left(\sum_{t=0}^{T-1}\nabla^2\log\pi_\theta(A_t | S_t)+
        \left[\sum_{t=0}^{T-1}\nabla\log\pi_\theta(A_t | S_t)\right] \left[\sum_{t=0}^{T-1}\nabla\log\pi_\theta(A_t | S_t)\right]^T\right)\right\rVert\right]\nonumber\\
        &\leq M_e M_h +M_e^2M_d^2, \label{eq:nabla2_F_R_bound}
    \end{align}
    where these inequalities follow from \eqref{eq:nabla_G_bound}, \eqref{eq:nabla_G_bound1}, and the assumption that the state and action spaces are finite.
\end{proof}
The following lemma establishes Lipschitzness of the CDF and its gradient.
\begin{lemma}
    \label{lm:F_lip}
    $\forall x \in(-M_r,M_r)$,
    \begin{align*}
        &\left\lvert F_{R^{\theta_1}}(x) - F_{R^{\theta_2}}(x) \right\rvert \leq M_eM_d \left\lVert \theta_1 - \theta_2 \right\rVert, \textrm{ and}\\
        &\left\lVert \nabla F_{R^{\theta_1}}(x) - \nabla F_{R^{\theta_2}}(x) \right\rVert \leq (M_eM_h+M_e^2M_d^2) \left\lVert \theta_1 - \theta_2 \right\rVert.
    \end{align*}
\end{lemma}
\begin{proof}
    The result follows by Lemma \ref{lm:nablaF_bound} and Lemma~1.2.2 in \cite{nesterov_book}.
\end{proof}
\subsubsection{Gradient of the DRM}
The following lemma derives an expression for the gradient of the DRM. This lemma is similar to Theorem 1 in \cite{nv2021}. For the sake of completeness, we provide detailed proof.
\begin{lemma}
    \label{lm:nabla_rho_g}
    $\nabla \rho_g(\theta)\!=\!-\int_{-M_r}^{M_r} g'(1-F_{R^{\theta}}(x)) \nabla F_{R^{\theta}}(x)dx$.
\end{lemma}
\begin{proof}
    Notice that
    \begin{align*}
        \nabla \rho_g(\theta)
        &=\nabla\int_{-M_r}^{0}\left(g(1-F_{R^{\theta}}(x))-1\right)dx  + \nabla\int_{0}^{M_r} g(1-F_{R^{\theta}}(x))dx \nonumber\\
        &\stackrel{(a)}{=}\int_{-M_r}^{0}\nabla \left(g(1-F_{R^{\theta}}(x))-1\right) dx + \int_{0}^{M_r}\nabla g(1-F_{R^{\theta}}(x)) dx \\
        &=-\int_{-M_r}^{M_r} g'(1-F_{R^{\theta}}(x)) \nabla F_{R^{\theta}}(x)dx.\nonumber
    \end{align*}
    In the above, \((a)\) follows by an application of the dominated convergence theorem to interchange the differentiation and the integration operation. The aforementioned application is allowed since
    (i) $\rho_g(\theta)$ is finite for any $\theta \in \R^d$; (ii)  $\lvert g'(\cdot) \rvert \leq M_{g'}$ from \ref{as:g'_bound}, and $\nabla F_{R^{\theta}}(\cdot)$ is bounded from \eqref{eq:nabla_F_R_bound}. The bounds on $g'$ and $\nabla F_{R^{\theta}}$ imply \\$\int_{-M_r}^{M_r} \left\lVert g'(1-F_{R^{\theta}}(x)) \nabla F_{R^{\theta}}(x)\right\rVert dx \leq 2 M_rM_{g'} M_eM_d$.
\end{proof}
\subsubsection{Lipschitz properties of the DRM and its gradient}
The following two lemmas establish the Lipschitzness of the DRM and its gradient.
\begin{proof}(\textbf{Lemma \ref{lm:rho_lip}})
    \begin {align*}
    \left\lvert \rho_g(\theta_1)- \rho_g(\theta_2)\right\rvert
    & \leq  \int_{-M_r}^{M_r}\left\lvert g(1-F_{R^{\theta_1}}(x))-g(1-F_{R^{\theta_2}}(x)) \right\rvert dx \\
    & \stackrel{(a)}{\leq} M_{g'} \int_{-M_r}^{M_r}\left\lvert F_{R^{\theta_1}}(x)-F_{R^{\theta_2}}(x)\right\rvert dx\\
    &\stackrel{(b)}{\leq} 2M_rM_{g'}M_eM_d \left\lVert \theta_1 - \theta_2 \right\rVert,
    \end {align*}
    where \((a)\) follows from Lemma \ref{lm:g_lip} and \((b)\) follows from Lemma \ref{lm:F_lip}. The result follows since $L_\rho=2M_rM_{g'}M_eM_d$.

    From Lemma \ref{lm:nabla_rho_g}, we obtain
    \begin {align*}
    &\left\lVert \nabla\rho_g(\theta_1) - \nabla \rho_g(\theta_2) \right\rVert \\
    &\leq \int_{-M_r}^{M_r}\left\lVert  g'(1-F_{R^{\theta_1}}(x)) \nabla F_{R^{\theta_1}}(x)
    - g'(1-F_{R^{\theta_2}}(x)) \nabla F_{R^{\theta_2}}(x) \right\rVert dx\\
    &\leq \int_{-M_r}^{M_r}\left\lVert  g'(1-F_{R^{\theta_1}}(x)) \nabla F_{R^{\theta_1}}(x)\right.
    -g'(1-F_{R^{\theta_1}}(x)) \nabla F_{R^{\theta_2}}(x) + g'(1-F_{R^{\theta_1}}(x)) \nabla F_{R^{\theta_2}}(x) \\
    &\quad \left.-g'(1-F_{R^{\theta_2}}(x)) \nabla F_{R^{\theta_2}}(x) \right\rVert dx   \\
    &\leq \int_{-M_r}^{M_r}  \left\lvert g'(1-F_{R^{\theta_1}}(x)) \right\rvert  \left\lVert \nabla F_{R^{\theta_1}}(x) - \nabla F_{R^{\theta_2}}(x)\right\rVert
    +\left\lVert\nabla F_{R^{\theta_2}}(x)\right\rVert \left\lvert  g'(1-F_{R^{\theta_1}}(x)) - g'(1-F_{R^{\theta_2}}(x))\right\rvert dx   \\
    &\stackrel{(a)}{\leq} \int_{-M_r}^{M_r} M_{g'} \left\lVert \nabla F_{R^{\theta_1}}(x) -  \nabla F_{R^{\theta_2}}(x)\right\rVert
    +M_eM_d M_{g''}\left\lvert F_{R^{\theta_1}}(x) - F_{R^{\theta_2}}(x)\right\rvert dx \\
    &\stackrel{(b)}{\leq} \int_{-M_r}^{M_r} M_{g'} (M_eM_h+M_e^2M_d^2)\left\lVert  \theta_1  - \theta_2\right\rVert
    + M_e^2M_d^2 M_{g''}\left\lVert  \theta_1  - \theta_2\right\rVert dx \\
    &\leq 2M_r M_e \left( M_h M_{g'}+M_eM_d^2 (M_{g'}+ M_{g''})\right) \left\lVert  \theta_1  - \theta_2\right\rVert,
    \end {align*}
    where \((a)\) follows from \ref{as:g'_bound}, and Lemmas \ref{lm:g_lip}, \ref{lm:nablaF_bound}, and \((b)\) follows from Lemma \ref{lm:F_lip}. The result follows since $L_{\rho'}=2M_r M_e \left( M_h M_{g'}+M_eM_d^2 (M_{g'}+ M_{g''})\right)$.
\end{proof}

\section{Mean-variance risk measure}
\label{sec:mvrm_proofs}
\subsection{The estimation error of the MVRM}
In the following lemma, we bound the estimation error of the MVRM in an on-policy RL setting.
\begin{proof}(\textbf{Lemma \ref{lm:est_error_mvrm_pi}})
    From \eqref{eq:rho_lambda} and \eqref{eq:hat_rho_lambda_pi}, we obtain
    \begin{align}
        \label{eq:mse_hat_rho_lambda_pi}
        \E\left[\left\lvert \hat{\rho}_\lambda^{\pi}(\theta)- \rho_\lambda(\theta) \right\rvert^2\right]
        &\leq2\E\left[\left\lvert \hat{J}_m^{\pi}(\theta)-J(\theta)\right\rvert^2\right]+2\lambda^2\E\left[\left\lvert V(\theta)-\widehat{V}_m^{\pi}(\theta)
        \right\rvert^2\right]\nonumber\\
        &\leq \frac{8M_r^2}{m}+\frac{32\lambda^2M_r^4}{m}=\frac{8M_r^2+32\lambda^2M_r^4}{m},
    \end{align}
    where the last inequality follows from Theorem 2-3 \cite[chapter V1]{mood74} in conjunction with the fact $\left\lvert R^\theta\right\rvert\leq M_r$ and $m>2$.
\end{proof}
In the following lemma, we bound the estimation error of the MVRM in an off-policy RL setting.
\begin{proof}(\textbf{Lemma \ref{lm:est_error_mvrm_b}})
    From \eqref{eq:rho_lambda} and \eqref{eq:hat_rho_lambda_b}, we obtain
    \begin{align}
        \label{eq:mse_hat_rho_lambda_b}
        \E\left[\left\lvert \hat{\rho}_\lambda^{b}(\theta)- \rho_\lambda(\theta) \right\rvert^2\right]
        &\leq2\E\left[\left\lvert \hat{J}_m^{b}(\theta)-J(\theta)\right\rvert^2\right]+2\lambda^2\E\left[\left\lvert V(\theta)-\widehat{V}_m^{b}(\theta)
        \right\rvert^2\right]\nonumber\\
        &\leq \frac{8M_r^2M_s^2}{m}+\frac{32\lambda^2M_r^4M_s^4}{m}=\frac{8M_r^2M_s^2+32\lambda^2M_r^4M_s^4}{m},
    \end{align}
    where the last inequality follows from Theorem 2-3 \cite[chapter V1]{mood74} in conjunction with the fact $\left\lvert R^b\psi_\theta\right\rvert\leq M_rM_s$, and $m>2$.
\end{proof}
\subsection{Lipschitz properties of the MVRM and its gradient}
\label{subsec:lip_mvrm}
\begin{proof}(\textbf{Lemma \ref{lm:lip_rho_lambda}})
    Let $\Omega$ denote the set of all sample episodes. For any episode $\omega\in\Omega$, we denote by $T(\omega)$, its length, and $S_t(\omega)$ and $A_t(\omega)$, the state and action at time $t\in\{0,1,2,\cdots\}$ respectively.

    Let $R(\omega)=\sum\limits_{t=0}^{T(\omega)-1}\gamma^t r(S_t(\omega),A_t(\omega),S_{t+1}(\omega))$ be the cumulative discounted reward of the episode $\omega$, and let \\
    $\p_\theta(\omega) = \prod\limits_{t=0}^{T(\omega)-1}\pi_\theta(A_t(\omega)|S_t(\omega))p(S_{t+1}(\omega),S_t(\omega),A_t(\omega))$. \\
    From $\frac{\nabla \p_\theta(\omega)}{\p_\theta(\omega)} =\sum\limits_{t=0}^{T(\omega)-1}\!\nabla\log\pi_\theta(A_t(\omega) | S_t(\omega))$,
    we obtain
    \begin{align}
        \label{eq:nabla_J}
        \nabla J(\theta)
        &=\nabla\E\left[R^{\theta}\right] =\nabla \sum_{\omega\in\Omega} R(\omega)\p_\theta(\omega)\nonumber\\
        &\stackrel{(a)}{=}\sum_{\omega\in\Omega}\nabla\left( R(\omega)\p_\theta(\omega)\right)\nonumber\\
        &\stackrel{(b)}{=}\sum_{\omega\in\Omega} R(\omega)\nabla \p_\theta(\omega)\nonumber\\
        &=\sum_{\omega\in\Omega} R(\omega) \frac{\nabla \p_\theta(\omega)}{\p_\theta(\omega)}\p_\theta(\omega)\\
        &=\sum_{\omega\in\Omega} R(\omega) \sum_{t=0}^{T(\omega)-1}\nabla\log\pi_\theta(A_t(\omega)|S_t(\omega))\p_\theta(\omega)\nonumber\\
        &=\E\left[R^{\theta}\sum_{t=0}^{T-1}\nabla\log\pi_\theta(A_t|S_t)\right].
    \end{align}
    In the above, \((a)\) follows by an application of the dominated convergence theorem to interchange the differentiation and the expectation operation. The aforementioned application is allowed since  (i) $\Omega$ is finite and the underlying measure is bounded, as we consider an MDP where the state and actions spaces are finite, and the policies are proper, (ii) $\nabla\log\pi_\theta(A_t|S_t)$ is bounded from \ref{as:nabla_logpi}. The step \((b)\) follows, since for a given episode $\omega$, the cumulative reward $R(\omega)$ does not depend on $\theta$.

    Similarly, \\
    from $\frac{\nabla^2 \p_\theta(\omega)}{\p_\theta(\omega)} =\sum\limits_{t=0}^{T(\omega)-1}\nabla^2\log\pi_\theta(A_t(\omega) | S_t(\omega))+
    \left[\sum\limits_{t=0}^{T(\omega)-1}\nabla\log\pi_\theta(A_t(\omega) | S_t(\omega))\right] \left[\sum\limits_{t=0}^{T(\omega)-1}\nabla\log\pi_\theta(A_t(\omega) | S_t(\omega))\right]^T$, we obtain
    \begin{align}
        \label{eq:nabla2_J}
        &\nabla^2 J(\theta)
        =\E\left[R^{\theta}\left(\sum_{t=0}^{T-1}\nabla^2\log\pi_\theta(A_t | S_t)+
        \left[\sum_{t=0}^{T-1}\nabla\log\pi_\theta(A_t | S_t)\right] \left[\sum_{t=0}^{T-1}\nabla\log\pi_\theta(A_t | S_t)\right]^T\right)\right].
    \end{align}

    Similarly,
    \begin{align}
        \label{eq:nabla_E_R_theta_sq}
        \nabla\E\left[\left(R^{\theta}\right)^2\right]
        & =\nabla \sum_{\omega\in\Omega} R(\omega)^2\p_\theta(\omega)\nonumber\\
        &=\sum_{\omega\in\Omega}\nabla\left( R(\omega)^2\p_\theta(\omega)\right)\\
        &=\sum_{\omega\in\Omega} R(\omega)^2\nabla \p_\theta(\omega)\nonumber\\
        &=\sum_{\omega\in\Omega} R(\omega)^2 \frac{\nabla \p_\theta(\omega)}{\p_\theta(\omega)}\p_\theta(\omega)\\
        &=\sum_{\omega\in\Omega} R(\omega)^2 \sum_{t=0}^{T(\omega)-1}\nabla\log\pi_\theta(A_t(\omega)|S_t(\omega))\p_\theta(\omega)\nonumber\\
        &=\E\left[\left(R^{\theta}\right)^2\sum_{t=0}^{T-1}\nabla\log\pi_\theta(A_t|S_t)\right],
    \end{align}
    and
    \begin{align}
        \label{eq:nabla2_E_R_theta_sq}
        &\nabla^2 \E\left[\left(R^{\theta}\right)^2\right]
        =\E\left[\left(R^{\theta}\right)^2\left(\sum_{t=0}^{T-1}\!\nabla^2\log\pi_\theta(A_t | S_t)+
        \left[\sum_{t=0}^{T-1}\!\nabla\log\pi_\theta(A_t | S_t)\right] \left[\sum_{t=0}^{T-1}\!\nabla\log\pi_\theta(A_t | S_t)\right]^T\right)\right].
    \end{align}
    From \eqref{eq:nabla_J}-\eqref{eq:nabla2_J}, we obtain
    \begin{align}
        \label{eq:norm_nabla_J}
        \left\lVert\nabla J(\theta)\right\rVert
        &\leq\E\left[ \left\lVert R^{\theta}\sum_{t=0}^{T-1}\nabla\log\pi_\theta(A_t|S_t)\right\rVert\right]
        \leq M_r E\left[ \left\lVert \sum_{t=0}^{T-1}\nabla\log\pi_\theta(A_t|S_t)\right\rVert\right]
        \leq M_r M_eM_d,
    \end{align}
    and
    \begin{align}
        \label{eq:norm_nabla2_J}
        \left\lVert\nabla^2 J(\theta)\right\rVert
        &\leq\E\left[\left\lVert R^{\theta}\left(\sum_{t=0}^{T-1}\nabla^2\log\pi_\theta(A_t | S_t)+
        \left[\sum_{t=0}^{T-1}\nabla\log\pi_\theta(A_t | S_t)\right] \left[\sum_{t=0}^{T-1}\nabla\log\pi_\theta(A_t | S_t)\right]^T\right)\right\rVert\right]\nonumber\\
        &\leq M_r\E\left[\left\lVert\sum_{t=0}^{T-1}\nabla^2\log\pi_\theta(A_t | S_t)\right\rVert+
        \left\lVert\sum_{t=0}^{T-1}\nabla\log\pi_\theta(A_t | S_t)\right\rVert^2\right]\nonumber\\
        &\leq M_r\left(M_eM_h+M_e^2M_d^2\right).
    \end{align}
    Hence from \eqref{eq:norm_nabla2_J} and Lemma~1.2.2 in \cite{nesterov_book}, we obtain
    \begin{align}
        \label{eq:lip_nabla2_J}
        \left\lVert \nabla J(\theta_1)-\nabla J(\theta_2)\right\rVert
        &\leq M_r\left(M_eM_h+M_e^2M_d^2\right) \left\lVert \theta_1 - \theta_2 \right\rVert
    \end{align}
    Similarly, from \eqref{eq:nabla_E_R_theta_sq}-\eqref{eq:nabla2_E_R_theta_sq}, we obtain
    \begin{align}
        \label{eq:norm_nabla_E_R_theta_sq}
        \left\lVert \nabla\E\left[\left(R^{\theta}\right)^2\right] \right\rVert
        &\leq M_r^2\E\left[\left\lVert\sum_{t=0}^{T-1}\nabla\log\pi_\theta(A_t|S_t)\right\rVert\right]
        \leq M_r^2M_eM_d,
    \end{align}
    and
    \begin{align}
        \label{eq:norm_nabla2_E_R_theta_sq}
        &\left\lVert\nabla^2 \E\left[\left(R^{\theta}\right)^2\right]\right\rVert
        \leq M_r^2\E\left[\left\rVert\sum_{t=0}^{T-1}\nabla^2\log\pi_\theta(A_t | S_t)\right\rVert+
        \left\lVert\sum_{t=0}^{T-1}\nabla\log\pi_\theta(A_t | S_t)\right\rVert^2\right]
        \leq M_r^2\left(M_eM_h+M_e^2M_d^2\right).
    \end{align}
    Now,
    \begin{align}
        \label{eq:norm_nabla2_V}
        \left\lVert \nabla^2 V(\theta)\right\rVert
        &=\left\lVert \nabla^2 \left(\E\left[\left(R^{\theta}\right)^2\right]-J(\theta)^2\right)\right\rVert\nonumber\\
        &=\left\lVert \nabla^2 \E\left[\left(R^{\theta}\right)^2\right]-2J(\theta) \nabla^2 J(\theta)-2\nabla J(\theta) \nabla J(\theta)^\top\right\rVert\nonumber\\
        &\leq\left\lVert \nabla^2 \E\left[\left(R^{\theta}\right)^2\right]\right\rVert+2\left\lvert J(\theta)\right\rvert \left\lVert\nabla^2 J(\theta)\right\rVert+2\left\lVert\nabla J(\theta)\right\rVert^2\nonumber\\
        &\leq 3M_r^2M_eM_h+5 M_r^2M_e^2M_d^2.
    \end{align}
    Hence, from \eqref{eq:norm_nabla2_V} and Lemma~1.2.2 in \cite{nesterov_book}, we obtain
    \begin{align}
        \label{eq:lip_nabla_V}
        \left\lVert \nabla V(\theta_1)-\nabla V(\theta_2)\right\rVert
        &\leq \lambda \left(3M_r^2M_eM_h+5 M_r^2M_e^2M_d^2\right) \left\lVert \theta_1 - \theta_2 \right\rVert
    \end{align}
    Now,
    \begin{align}
        \label{eq:norm_nabla_rho_lambda}
        \left\lVert \nabla \rho_\lambda(\theta)\right\rVert&=
        \left\lVert \nabla J(\theta)-\lambda\nabla V(\theta)\right\rVert\nonumber\\
        &\leq\left\lVert\nabla J(\theta)\right\rVert+\lambda\left\lVert\nabla \E\left[\left(R^{\theta}\right)^2\right]\right\rVert+2\lambda \left\lvert J(\theta) \right\rvert \left\lVert\nabla J(\theta)\right\rVert\nonumber\\
        &\leq M_rM_eM_d+ 3\lambda M_r^2 M_e M_d.
    \end{align}
    Hence, from \eqref{eq:norm_nabla_rho_lambda} and Lemma~1.2.2 in \cite{nesterov_book}, we obtain
    \begin{align}
        \label{eq:lip_rho_lambda}
        \left\lvert \rho_\lambda(\theta_1)-\rho_\lambda(\theta_1)\right\rvert
        &\leq \left(M_rM_eM_d+ 3\lambda M_r^2 M_e M_d\right)  \left\lVert \theta_1 - \theta_2 \right\rVert.
    \end{align}
    From \eqref{eq:lip_nabla2_J} and \eqref{eq:lip_nabla_V}, we obtain
    \begin{align}
        \label{eq:lip_nabla_rho_lambda}
        \left\lVert \nabla \rho_\lambda(\theta_1)-\nabla \rho_\lambda(\theta_1)\right\rVert &\leq
        \left\lVert \nabla J(\theta_1)-\nabla J(\theta_2) \right\rVert+ \lambda\left\lVert \nabla V(\theta_2)-\nabla V(\theta_1)\right\rVert\nonumber\\
        &\leq \left( M_rM_e\left(M_h+M_eM_d^2\right)+\lambda M_r^2M_e\left(3M_h+5 M_eM_d^2\right)\right) \left\lVert \theta_1 - \theta_2 \right\rVert.
    \end{align}
\end{proof}



\end{document}